\documentclass[journal]{IEEEtran} 
\IEEEoverridecommandlockouts                         
\usepackage{pifont}
\usepackage{multirow}
\usepackage{graphicx}
\usepackage{multicol}
\usepackage{booktabs}
\usepackage{amsmath,amssymb}
\usepackage{amsthm}
\usepackage[utf8]{inputenc}
\usepackage[english]{babel}
\usepackage{multirow}
\usepackage[font=small]{caption}
\usepackage{float}
\usepackage[hidelinks]{hyperref}

\usepackage{lettrine}
\usepackage[]{algorithm2e}
\usepackage{algpseudocode}
\usepackage{cite}
\usepackage{lipsum}
\usepackage{color}
\usepackage[dvipsnames]{xcolor}
\usepackage{esdiff}
\usepackage{epstopdf}
\usepackage[normalem]{ulem}
\usepackage{soul}

\usepackage{subcaption}
\usepackage{xcolor}
\usepackage{colortbl}
\usepackage{balance}
\newtheorem{thm}{Theorem}[section]

\newtheorem{lem}{Lemma}[section]

\newtheorem{defn}{Definition}[section]

\newtheorem{rem}{Remark}[section]

\definecolor{pink}{rgb}{1, 0, 1}
\definecolor{orange}{rgb}{1, 0.7529, 0}
\definecolor{darkgreen}{rgb}{0, 0.8, 0}
\begin{document}

\title{TRACE: Coverage Path Planning for Unknown Environments Using Hierarchical Coverage Tree}
\author{
{Zongyuan Shen$^1$}, {Haodong Liu$^1$}, {Gao Wang$^1$}, {Shancheng Zhao$^1$}, {Dehua Zhou$^1$}, {Yaming Ou$^2$},\\{Zhongqiang Ren$^3$}, {Yikui Zhai$^4$}, and {C. L. Philip Chen$^5$, \textit{Life Fellow, IEEE}}

\thanks {$^1$College of Information Science and Technology, Jinan University, Guangzhou 510632, China.}
\thanks {$^2$School of Artificial Intelligence, University of Chinese Academy of Sciences, Beijing 100049, China.}
\thanks {$^3$Global College, Shanghai Jiao Tong University, Shanghai 200240, China.}
\thanks {$^4$School of Electronics and Information Engineering, Wuyi University, Jiangmen 529000, China.}
\thanks {$^5$School of Computer Science and Engineering, South China University of Technology, Guangzhou 510006, China.}
}

\maketitle

\begin{abstract}
This paper presents a novel online coverage path planning (CPP) algorithm, called TRACE, for real-time coverage of unknown environments. TRACE is built upon a hierarchical coverage tree that provides a global representation of the evolving connectivity of the uncovered space. As the environment is incrementally revealed and covered, newly discovered obstacles and covered cells may fragment the remaining uncovered space into disconnected regions. TRACE recursively expands the corresponding tree nodes to explicitly represent these regions and organize them for subsequent coverage planning. Based on the updated tree, an incremental global tour is maintained to guide the coverage process. TRACE locally refines only the affected portions while preserving the visiting order of unchanged regions, thereby reducing the computational burden of global replanning and maintaining a consistent coverage progression. Guided by the global tour, a local planner generates back-and-forth coverage paths and switches to global-tour-aware planning to efficiently complete the target regions. Theoretical analysis establishes the computational complexity and complete coverage property of TRACE, and derives an approximation bound for the incremental global tour refinement. The performance of TRACE is evaluated through extensive high-fidelity simulations and real-robot experiments using a mobile robot. Comparative evaluations against six existing CPP methods demonstrate significant improvements in coverage time, path length, overlap ratio, and number of turns.
\end{abstract}


\section{Introduction}
Coverage path planning (CPP) is a fundamental problem in robotics that aims to generate a trajectory allowing a robot to completely cover a target workspace while minimizing task-related costs such as path length, overlap ratio, number of turns, and coverage time~\cite{shen2026cstar}. CPP has a broad range of applications, including agricultural operations (e.g., weeding~\cite{maini2022online} and harvesting~\cite{yi2024view}); household services (e.g., floor cleaning~\cite{ramesh2024} and lawn mowing~\cite{zhou2025coverage}); environmental monitoring (e.g., underwater mapping~\cite{Ou2025} and oil spill cleaning~\cite{luo2025entropy}); and industrial tasks (e.g., quality inspection~\cite{wang2023hierarchical}, crack filling~\cite{veeraraghavan2024complete}, and spray painting~\cite{yang2024improved}). In many practical applications, however, complete environmental information is unavailable prior to deployment. Therefore, it is essential to develop online CPP methods that can incrementally discover the workspace through onboard sensing while simultaneously adapting the coverage path according to newly acquired information. 

Various methods have been developed for online CPP. One line of works generates coverage trajectories by selecting the next target cell in the robot's local neighborhood using lightweight deterministic rules~\cite{gabriely2001spanning,gonzalez2005bsa,ferranti2007brick,viet2013ba}, potential fields~\cite{luo2008bioinspired,sun2019complete,cai2023,song2018}, or reward functions~\cite{hassan2019ppcpp}. Although such strategies are computationally efficient, they do not explicitly account for the global structure and evolution of the remaining uncovered space. In complex environments, locally reasonable decisions may fragment the remaining uncovered space into disconnected regions, causing some regions to be bypassed and later revisited through long return paths. To mitigate this problem, some methods~\cite{li2023sp2e,huo2025,li2025} incorporate connectivity preservation into target selection by evaluating whether visiting a candidate cell would disconnect the remaining uncovered area. By avoiding connectivity-breaking decisions, these methods reduce fragmentation-induced redundant travel. However, they still rely on local feasibility checks and may fail when all locally available candidates would disconnect the uncovered area. In such cases, fragmented residual regions may still emerge, resulting in considerable redundant travel. Another line of works explicitly handles fragmented regions after they emerge, either by organizing disconnected regions within the currently obstacle-free space for subsequent coverage~\cite{shen2025cap} or by covering potentially isolated regions near the robot before they are bypassed~\cite{shen2026cstar}. However, their detection and handling of fragmentation are still largely driven by local environmental information. As a result, fragmented regions arising away from robot vicinity may not be explicitly considered when they emerge, which can lead to additional backtracking and reduced coverage efficiency in complex environments.

In this regard, this paper presents TRACE, a novel online CPP algorithm that maintains a global representation of the evolving uncovered space using a hierarchical coverage tree, which tracks connectivity changes of the remaining uncovered regions and guides the subsequent coverage process. As the environment is incrementally revealed and covered, newly discovered obstacles and newly covered cells may alter the connectivity of the residual region associated with a leaf node. If a residual region is split into multiple disconnected components, the corresponding leaf node is expanded into multiple child nodes, each representing one of these components. Throughout this process, the tree progressively evolves from a coarse representation of the entire workspace into a hierarchical representation of the remaining uncovered regions. The resulting leaf nodes represent candidate regions to be considered in subsequent coverage planning. Based on the updated hierarchical coverage tree, a global tour is maintained to determine the visiting order of the remaining candidate regions. Rather than replanning the entire tour whenever the tree changes, TRACE incrementally refines only the portions affected by tree expansion. Specifically, when a leaf node is expanded, its position in the current tour is replaced by an optimized visiting sequence of the newly generated child nodes, while the relative order of the unaffected nodes is preserved. This incremental update reduces the computational burden of global replanning while avoiding unnecessary reordering of unaffected regions and maintaining a consistent coverage progression. Guided by the updated global tour, the local planner progressively covers the selected target region using a back-and-forth strategy and switches to a global-tour-aware strategy once the region becomes fully explored. This adaptive local planning scheme improves local coverage efficiency while maintaining consistency with the global tour. In summary, our contributions
are as follows:

\begin{itemize}
    \item A hierarchical coverage tree globally represents the evolving connectivity of the uncovered space and supports efficient coverage path planning.
    \item A global tour planning method locally refines changed tour segments while preserving unchanged visiting orders, providing global guidance for subsequent coverage.
    \item A local path planning method provides efficient dead-end recovery and global-tour-aware coverage path generation, reducing redundant travel and facilitating efficient transitions between successive target regions.
    \item Comparative evaluations of proposed algorithm through high-fidelity simulations and real-robot experiments.
\end{itemize}

The remaining paper is organized as follows. Section~\ref{sec:review} reviews the related work. Section~\ref{sec:problem_description} describes the CPP problem. Section~\ref{sec:algorithm} presents the details of the TRACE algorithm. Section~\ref{sec:analysis} provides the algorithm analysis. Section~\ref{sec:results} shows the comparative evaluation results of TRACE. Finally, Section~\ref{sec:conclusions} concludes the paper with recommendations for future work.

\section{Related Work}
\label{sec:review}

Detailed surveys of existing CPP methods are presented in~\cite{galceran2013survey,shen2026coverage}. CPP methods can be classified as offline or online according to the availability of environmental knowledge. Offline methods assume that the environment is known and generate coverage paths before deployment. For example, Cao et al.~\cite{cao2020hierarchical} sampled candidate viewpoints from a prior map and then locally selected viewpoints for detailed coverage, Shen et al.~\cite{shen2021cppnet} utilized a graph neural network to generate coverage paths from occupancy grid maps, Ramesh et al.~\cite{ramesh2022optimal} formulated turn-minimizing coverage as an optimization problem, and Lu et al.~\cite{lu2023tmstc} partitioned the workspace into rectangular subareas for coverage planning. However, their performance can degrade if the prior information
is incomplete and/or incorrect. In contrast, online CPP methods generate coverage paths in real time using environmental information acquired during navigation, thus they are suitable for unknown environments.

A large class of online CPP methods generates coverage paths by selecting target cells in the robot's local neighborhood. Gabriely and Rimon~\cite{gabriely2001spanning} proposed the spanning tree covering (STC) algorithm, which constructs a spanning tree over a two-resolution grid and generates a coverage path by circumnavigating the tree. Gonzalez et al.~\cite{gonzalez2005bsa} proposed the backtracking spiral algorithm (BSA), which follows the boundary of uncovered and obstacle regions to generate a spiral coverage path. Ferranti et al.~\cite{ferranti2007brick} developed the brick-and-mortar (BM) algorithm, which progressively enlarges inaccessible regions by selecting the neighboring uncovered cell with the most surrounding inaccessible cells. Viet et al.~\cite{viet2013ba} proposed the BA$^*$ algorithm, which follows a predefined directional priority to generate a back-and-forth coverage path. Luo and Yang~\cite{luo2008bioinspired} proposed the bio-inspired neural network (BINN) algorithm, where each grid cell is associated with a neural activity such that uncovered cells provide excitatory inputs while obstacle cells provide inhibitory inputs. The neighboring cell with the highest neural activity is selected as the target. Later variants~\cite{sun2019complete,cai2023} simplified the neural activity update to reduce computational time. Song and Gupta~\cite{song2018} proposed the $\varepsilon^*$ algorithm, which constructs a multi-level potential field for online coverage. The robot normally selects the neighboring cell with the highest potential at the lowest level and switches to higher-level fields when escaping dead-ends. Hassan and Liu~\cite{hassan2019ppcpp} proposed the predator-prey CPP (PPCPP) algorithm, which evaluates neighboring cells using a reward function combining predation avoidance, motion smoothness, and boundary coverage, and selects the cell with the highest reward as the target. These methods are computationally efficient and suitable for real-time operation; however, their decisions are mainly based on local information and may generate myopic decision, leading to dead-ends, fragmented uncovered regions, and redundant travel in complex environments. 

To improve the coverage efficiency, several methods incorporate the connectivity of uncovered region into coverage decision-making. Li et al.~\cite{li2023sp2e} adapted the spiral direction to select target cells that preserve the connectivity of the uncovered region, while Huo et al.~\cite{huo2025} similarly avoided target cells whose visitation would cause disconnection. Li et al.~\cite{li2025} extended this idea to the subarea level by partitioning the environment into stripe-shaped regions and prioritizing the coverage of subareas whose removal preserves the connectivity of the remaining uncovered space. However, these methods still rely on local feasibility checks. When all candidate targets would disconnect the uncovered area, fragmented regions may still be produced, leading to considerable redundant travel. Other methods instead handle disconnected regions once they arise. Shen et al.~\cite{shen2025cap} proposed the CAP algorithm, which identifies disconnected subareas around the robot and computes a global traversal tour to guide their subsequent coverage. Shen et al.~\cite{shen2026cstar} proposed the C$^*$ algorithm, which performs back-and-forth coverage while identifying potential disconnected regions near the robot that may otherwise be bypassed. These regions are covered in situ using TSP-based local trajectories before the robot resumes back-and-forth coverage. However, fragmented regions arising away from robot vicinity may not be explicitly considered when they emerge, which can lead to additional backtracking and reduced coverage efficiency in complex environments.
\section{Problem Description}\label{sec:problem_description}

As shown in Fig.~\ref{fig:overview}, consider a mobile robot equipped with:
\begin{itemize}
\item A range detector (e.g., LiDAR) to detect obstacles within its field of view and map the environment. 
\item A localization device (e.g., IMU) to provide position information of the robot.
\item A coverage device or sensor (e.g., cleaning brush) to perform the task (e.g., floor cleaning).
\end{itemize} 

Let $\mathcal{X} \subset\mathbb{R}^2$ be the initially unknown area whose boundary is defined either by a hard border (e.g., a wall) or a soft barrier (e.g., user-defined buffer zone). The area is populated by static obstacles. For environment mapping and coverage planning, a uniform tiling of $\mathcal{X}$ is constructed as defined below.

\begin{defn}[Uniform Tiling]
\label{def:tiling}
A set $\mathcal{T}=\{\tau_j\subset\mathbb{R}^{2}, \ j=1,\ldots,|\mathcal{T}|\}$ is called a uniform tiling of $\mathcal{X}$ if its elements, called cells, satisfy the following conditions:
\begin{itemize}
    \item Each cell $\tau_j$ is a square of identical size;
    \item The interiors of any two distinct cells do not overlap, i.e., $\operatorname{int}(\tau_i) \cap \operatorname{int}(\tau_j)=\emptyset,\ i\neq j$;
    \item The area $\mathcal{X}$ is contained in the union of all cells, i.e., $\mathcal{X} \subseteq \bigcup_{\tau_j\in\mathcal{T}}\tau_j$.
\end{itemize}
\end{defn}

\begin{defn}[Path-Connected Region]\label{define:pathconnected}
A region $\mathcal{S} \subseteq \mathcal{T}$ is said to be path-connected if, for any two cells $\tau_i,\tau_j \in \mathcal{S}$, there exists a collision-free path consisting entirely of cells in $\mathcal{S}$ that connects $\tau_i$ and $\tau_j$.
\end{defn}

Let $\mathcal{T}_{o}\subset\mathcal{T}$ denote the subset of cells occupied by obstacles. The obstacle-free cell set $\mathcal{T}_{f}=\mathcal{T}\setminus\mathcal{T}_{o}$ constitutes the path-connected coverage space to be covered by the robot.

\begin{defn}[Coverage Path]
\label{def:coverage_path}
Let $\tau(k)\in\mathcal{T}_f$ denote the obstacle-free cell visited by the robot at discrete step $k\in\mathbb{Z}_{\geq 0}$. A coverage
path is defined as the ordered cell sequence $\Gamma_K=[\tau(0),\tau(1),\ldots,\tau(K)]$, where $\tau(0)$ and $\tau(K)$ denote the start and end cells, respectively.
\end{defn}

Let $\mathcal{C}(\tau(k)) \subset \mathcal{T}_f$ denote the set of cells covered by the
coverage device when the robot visits cell $\tau(k)$. The objective of the proposed algorithm is to generate a coverage path that achieves complete coverage of $\mathcal{T}_f$, as defined below. 

\begin{defn}[Complete Coverage]
\label{def:complete_coverage}
Given a coverage path $\Gamma_K=[\tau(0),\tau(1),\ldots,\tau(K)]$, the coverage of $\mathcal{T}_f$ is complete if there exists a finite $K\in\mathbb{Z}_{+}$ such that $\mathcal{T}_f
    \subseteq
    \bigcup_{k=0}^{K}\mathcal{C}(\tau(k))$.
\end{defn}
\section{TRACE Algorithm}\label{sec:algorithm}

This section presents the details of the proposed TRACE algorithm, which adaptively updates the coverage path as the environment is revealed through onboard sensing. An overview of the TRACE algorithm is shown in Fig.~\ref{fig:overview}. During navigation, the symbolic map is continuously updated according to sensor measurements and coverage progress (Sec.~\ref{mapping}). Based on the updated map, a hierarchical coverage tree is incrementally constructed to represent the remaining coverage regions and their connectivity changes (Sec.~\ref{tree_construct}). The active leaf nodes representing the remaining coverage regions are then extracted from the updated tree. The global visiting order of the active leaf nodes is maintained through incremental tour refinement, where only the portions affected by tree expansion are locally updated while the visiting order of unaffected regions is retained (Sec.~\ref{global_tour}). A target node is then selected from the updated global tour, and the local planner generates a coverage path within its subarea (Sec.~\ref{localCover}). These steps are repeated until complete coverage of the obstacle-free space is achieved.

\begin{figure}[t]
        \centering        \includegraphics[width=0.5\textwidth]{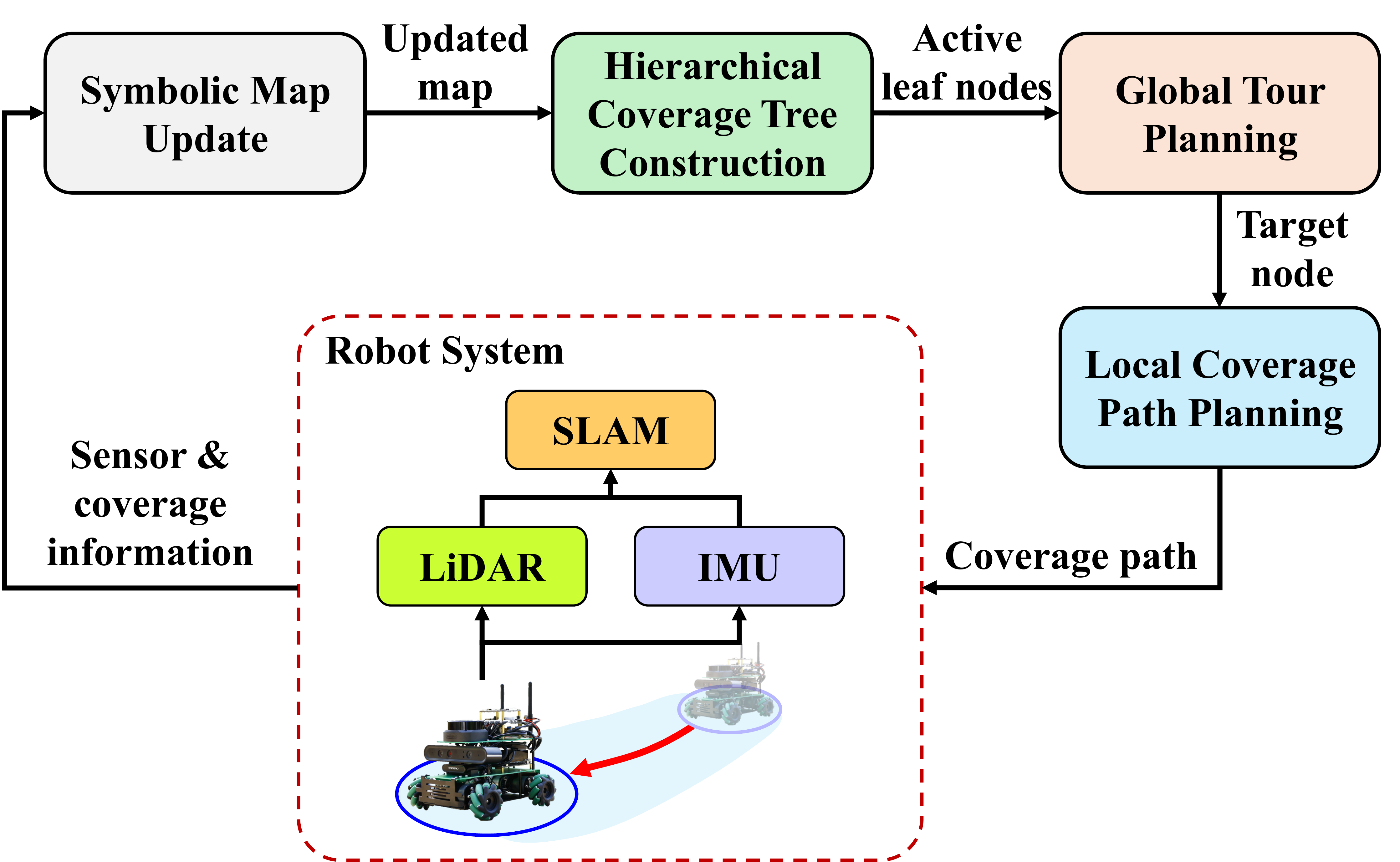}
    \caption{Overview of the proposed TRACE algorithm.}\label{fig:overview} 
    \vspace{-1.0em}
 \end{figure}

\subsection{Symbolic Map Update}
\label{mapping}

During navigation, the robot continuously updates a symbolic representation of the environment using onboard sensor measurements. This representation jointly encodes the occupancy and coverage status of each cell, as defined below.

\begin{defn}[State Encoding of Cell]
\label{stateencode_cell}
Let $\Phi: \mathcal{T} \rightarrow \{U,O,F^c,F^u\}$ be the symbolic state encoding of the tiling $\mathcal{T}$. All cells are initially assigned the state $U$. As environmental information is acquired during navigation, a cell occupied by an obstacle is assigned $O$. A free cell is assigned $F^c$ after it has been covered by the robot and $F^u$ otherwise.
\end{defn}

A morphological closing operator~\cite{soille2013} is applied to the symbolic map. This operation first dilates and then erodes the obstacle regions to smooth obstacle boundaries and suppress small-scale obstacles that are typically caused by sensor noise. The updated symbolic map provides the basis for the incremental construction of the hierarchical coverage tree.

\begin{figure*}[t]
    \centering
    \subfloat[Tree initialization.]{
        \includegraphics[width=0.32\textwidth]{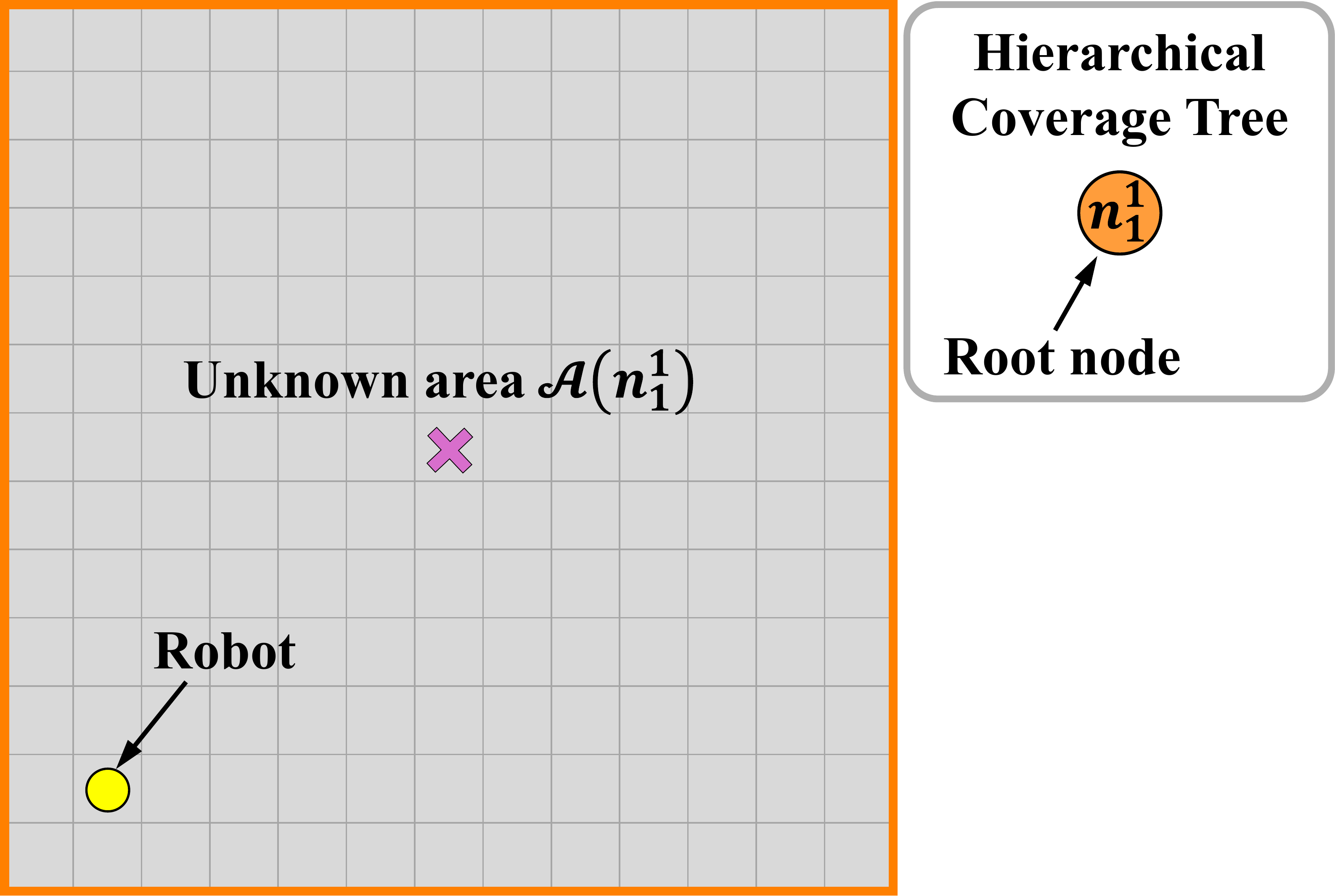}\label{fig:tree_example_part1}}\quad \hspace{-10pt}
        \centering
    \subfloat[First tree expansion and global tour planning.]{
        \includegraphics[width=0.32\textwidth]{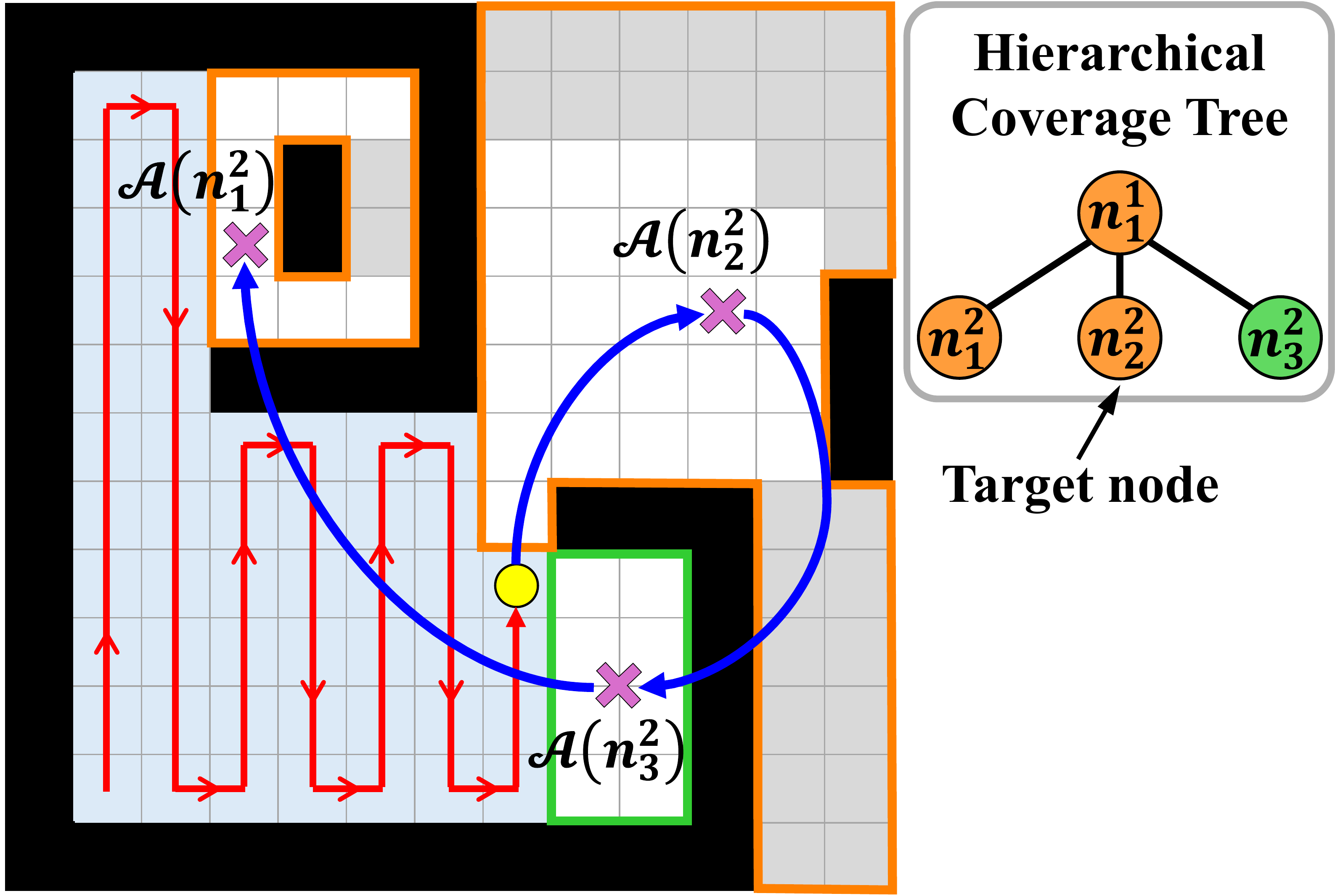}\label{fig:tree_example_part2}}\quad \hspace{-10pt}
    \centering
    \subfloat[Subsequent tree expansion and tour update.]{
        \includegraphics[width=0.32\textwidth]{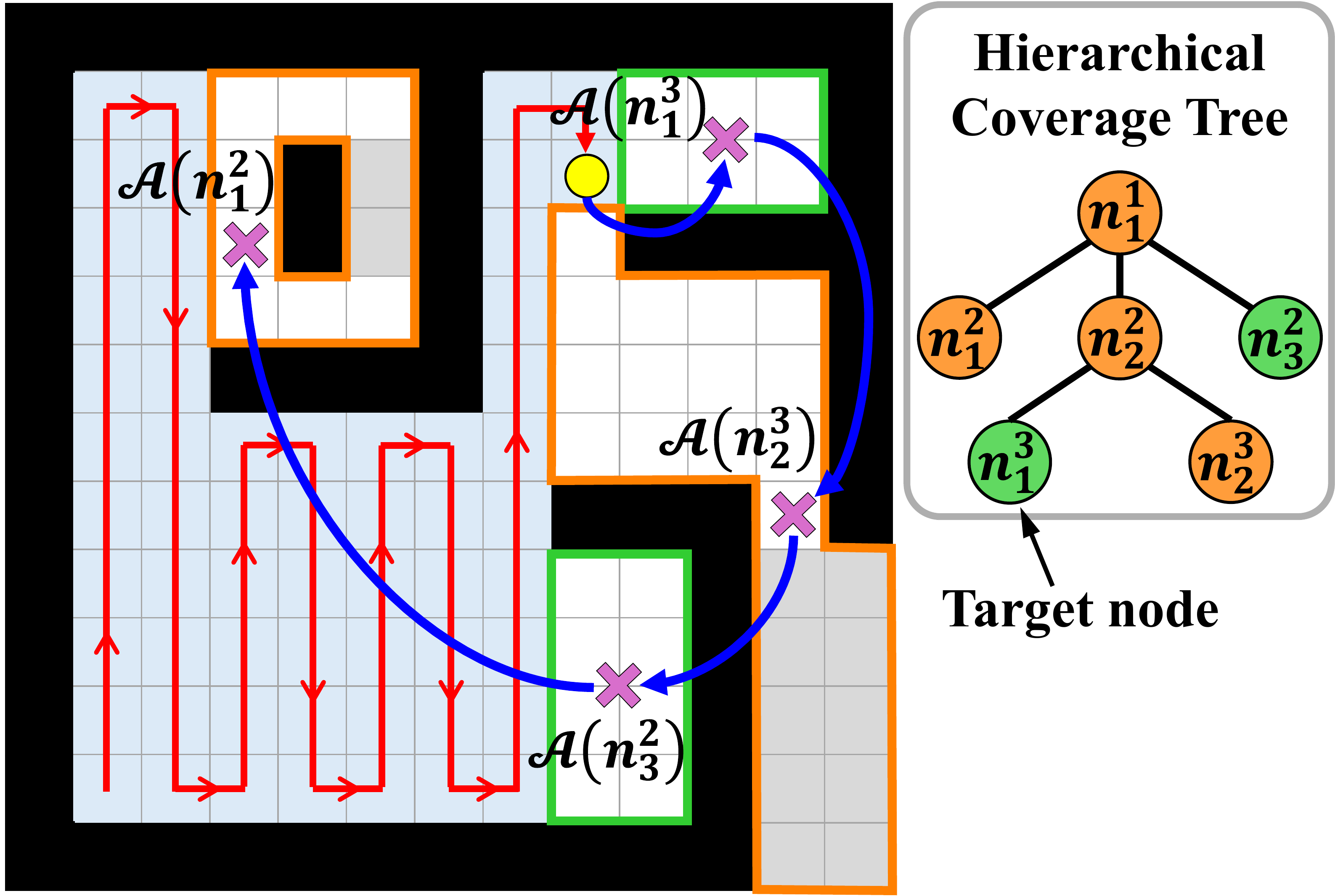}\label{fig:tree_example_part3}}\vspace{0.5em}\quad
        \centering
    \subfloat{
    \includegraphics[width=0.95\textwidth]{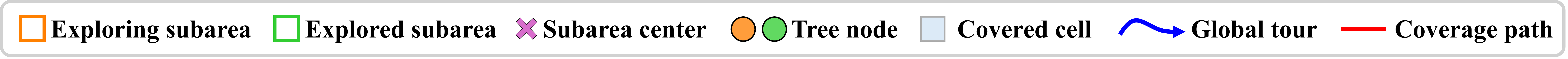}\vspace{0.5em}\label{fig:legend}}\\
    \caption{Illustration of hierarchical coverage tree evolution and incremental coverage planning. (a) The entire unknown workspace is initially represented by the root node $n_1^1$. (b) As the environment is revealed and covered, the residual region of $n_1^1$ is divided into three disconnected subareas, represented by child nodes $n_1^2$, $n_2^2$, and $n_3^2$. A global tour is then generated to guide their subsequent coverage. (c) With further environmental discovery and coverage progress, the residual region associated with $n_2^2$ is further divided into two disconnected subareas, represented by child nodes $n_1^3$ and $n_2^3$. The hierarchical coverage tree and global tour are updated accordingly.}\label{fig:algorithm_example} 
    \vspace{-1.0em}
 \end{figure*}

\subsection{Hierarchical Coverage Tree Construction}
\label{tree_construct}

The hierarchical coverage tree serves three main purposes: i) recording the hierarchical decomposition of the search space, ii) tracking coverage progress, and iii) supporting coverage decision-making. The tree is formally defined as follows.

\begin{defn}[Hierarchical Coverage Tree]\label{define:hierarchical_tree}
A hierarchical coverage tree $\mathcal{G} = (\mathcal{N}, \mathcal{E})$ is defined as a rooted tree, where:
\begin{itemize}
    \item For each level $\ell=1,\ldots,L$, where $L$ is the current depth of the tree, let $\mathcal{N}^{\ell} = \{n_i^{\ell}: i = 1,\ldots,I_{\ell}\}$ be the node set at level $\ell$, with $\mathcal{N}=\bigcup_{\ell=1}^{L}\mathcal{N}^{\ell}$. Each node $n_i^{\ell}$ maintains: 
    \begin{itemize}
        \item[a)] associated subarea $\mathcal{A}(n_i^{\ell})$, which is assigned when the node is created and remains fixed thereafter;
        \item[b)] residual region $\mathcal{R}(n_i^{\ell})=\{\tau \in \mathcal{A}(n_i^{\ell}):\Phi(\tau)\in\{U,F^u\} \}$, which contains only unknown and free-uncovered cells;
        \item[c)] center position $p(n_i^{\ell})$ of the residual region; and
        \item[d)] symbolic state $\mathcal{Q}(n_i^{\ell})$.
    \end{itemize}
    \item $\mathcal{E} = \{e_\alpha: \alpha = 1,\ldots,|\mathcal{E}|\}$ is the edge set. Each edge $e_{\alpha} \equiv (n_i^{\ell},n_{i'}^{\ell+1})$ connects a parent node $n_i^{\ell}$ at level $\ell$ to one of its child nodes $n_{i'}^{\ell+1}$ at level $\ell+1$.
\end{itemize}
\end{defn}

\begin{defn}[Exploring Region]\label{define:ExploringSubarea} A region $\mathcal{S} \subseteq \mathcal{T}$ is said to be \textit{exploring} if it contains at least one unknown cell, i.e., $\exists\tau \in \mathcal{S}$ such that $\Phi(\tau)=U$. Otherwise, it is \textit{explored}.
\end{defn}

\begin{defn}[State Encoding of Node]
\label{stateencode_node}
Let $\mathcal{Q}: \mathcal{N} \rightarrow \{C,E,\bar{E}\}$ assign a symbolic state $\mathcal{Q}(n_i^{\ell})$ to each node $n_i^{\ell} \in \mathcal{N}$. The state $C$ indicates that the residual region $\mathcal{R}(n_i^{\ell})$ is empty, and thus the corresponding coverage task has been completed. In contrast, $E$ and $\bar{E}$ indicate nodes whose residual regions are \textit{explored} and \textit{exploring}, respectively. 
\end{defn}

As shown in Fig.~\ref{fig:tree_example_part1}, since the environment is initially unknown, the tree is initialized with a single root node $n_1^1$, where $\mathcal{A}(n_1^1)=\mathcal{R}(n_1^1)=\mathcal{T}$. The center position $p(n_1^1)$ is computed as the average position of all cells in $\mathcal{A}(n_1^1)$. The symbolic state is set as $\mathcal{Q}(n_1^1)=\bar{E}$. Since the tree contains only the root node, the edge set is initially empty, i.e., $\mathcal{E}=\emptyset$. During online operation, the hierarchical coverage tree is incrementally updated through two steps: 1) node update and 2) tree expansion, as outlined in Alg.~\ref{alg:tree_update}. 

\begin{defn}[Changed Cell Set]
\label{changed_cell_set}
Let $\Delta \mathcal{T} \subseteq \mathcal{T}$ denote the set of cells whose symbolic states change during the latest map update, i.e., $\Delta \mathcal{T}=\{\tau_j \in \mathcal{T}:\Phi_{\mathrm{new}}(\tau_j) \neq \Phi_{\mathrm{old}}(\tau_j)\}$, where $\Phi_{\mathrm{old}}(\tau_j)$ and $\Phi_{\mathrm{new}}(\tau_j)$ are the symbolic maps before and after the latest map update, respectively. 
\end{defn}

\begin{defn}[Leaf Index Map]
\label{leaf_index_map}
Let $\mathcal{N}_{\mathrm{leaf}} \subseteq \mathcal{N}$ be the set of leaf nodes of the tree. Let $\Psi: \mathcal{T} \rightarrow \mathcal{N}_{\mathrm{leaf}} \cup \{\emptyset\}$ be an indexing function that associates each cell $\tau_j \in \mathcal{T}$ with the unique leaf node whose associated subarea contains the cell. Specifically, $\Psi(\tau_j)=n_i^{\ell}$ if $\tau_j \in \mathcal{A}(n_i^{\ell})$ and $n_i^{\ell} \in \mathcal{N}_{\mathrm{leaf}}$; otherwise, $\Psi(\tau_j)=\emptyset$. At tree initialization, all cells are indexed to the root node, i.e., $\Psi(\tau_j)=n_1^1,\forall\tau_j\in\mathcal T$.
\end{defn}

\subsubsection{Node Update} 
\label{node_update}

After each update of the symbolic map $\Phi$, the changed cells are mapped to their corresponding leaf nodes to identify the set of affected leaf nodes $\Delta\mathcal{N}_{\mathrm{leaf}}=\{\Psi(\tau_j):\tau_j \in \Delta\mathcal{T},\Psi(\tau_j) \neq \emptyset\}$ (Alg.~\ref{alg:tree_update}, Lines 1-6). The information of each node $n_i^{\ell} \in \Delta\mathcal{N}_{\mathrm{leaf}}$ is then updated according to the latest cell states (Alg.~\ref{alg:tree_update}, Line 8). Specifically, the residual region $\mathcal{R}(n_i^{\ell})$ is refined by removing obstacles or free-covered cells. The center $p(n_i^{\ell})$ is obtained by averaging the positions of all cells in $\mathcal{R}(n_i^{\ell})$. Due to irregular geometry, the resulting center may fall into obstacles or other regions. In such cases, it is projected to the position of the nearest cell within $\mathcal{R}(n_i^{\ell})$. Finally, the symbolic state $\mathcal{Q}(n_i^{\ell})$ is updated according to Defn.~\ref{stateencode_node}. A node is assigned $C$ if its residual region is empty; otherwise, it is assigned $E$ or $\bar{E}$ depending on whether the residual region is \textit{explored} or \textit{exploring}, respectively.

\subsubsection{Tree Expansion} 
\label{tree_expand}
Following node update, the connectivity of the residual regions of the affected leaf nodes may change, since newly identified obstacle cells and covered free cells have been removed. Therefore, each affected leaf node $n_i^{\ell} \in \Delta\mathcal{N}_{\mathrm{leaf}}$ with a nonempty residual region is examined to determine whether $\mathcal{R}(n_i^{\ell})$ remains path-connected.

For each such node, the process begins by selecting an unlabeled cell from $\mathcal{R}(n_i^{\ell})$ and assigning it a region index $j$. Starting from this cell, a flood-fill search recursively labels all four-connected neighboring cells within $\mathcal{R}(n_i^{\ell})$ using the same index $j$. The search terminates when no further eligible cells remain, and the labeled cells form a maximal path-connected component $\mathcal{S}_j$. This procedure is repeated until every cell in $\mathcal{R}(n_i^{\ell})$ has been labeled, resulting in $\mathcal{R}(n_i^{\ell}) = \bigcup_{j=1}^J \mathcal{S}_j$, where $\mathcal{S}_j \cap \mathcal{S}_{j'} = \emptyset, \;\forall \, j \neq j'\in \{1,\ldots,J\}$ (Alg.~\ref{alg:tree_update}, Line 10).

If $J=1$, $\mathcal{R}(n_i^{\ell})$ remains path-connected and the tree structure associated with $n_i^{\ell}$ remains unchanged. If $J>1$, $\mathcal{R}(n_i^{\ell})$ is separated into several disjoint portions. In this case, $J$ child nodes are created at level $\ell+1$, denoted by $\{n_{I_{\ell+1}+j}^{\ell+1}\}_{j=1}^J$, where $I_{\ell+1}$ is the number of existing nodes at level $\ell+1$ (Alg.~\ref{alg:tree_update}, Line 17). Each portion is assigned to one child node as its associated subarea and residual region, i.e., $\mathcal{A}(n_{I_{\ell+1}+j}^{\ell+1})=\mathcal{R}(n_{I_{\ell+1}+j}^{\ell+1})=\mathcal{S}_j$. The center of each child node is computed by following the same procedure described in the node update step, and its symbolic state is assigned according to Defn.~\ref{stateencode_node}. An edge $(n_i^{\ell},n_{I_{\ell+1}+j}^{\ell+1})$ is then added to $\mathcal{G}$ to connect the parent node $n_i^{\ell}$ to its child node $n_{I_{\ell+1}+j}^{\ell+1}$ (Alg.~\ref{alg:tree_update}, Line 19). The parent node \(n_i^{\ell}\) remains in the tree as an internal node, while its child nodes become new leaf nodes, thereby preserving the hierarchical decomposition history. The leaf index map $\Psi$ is updated accordingly. For each cell $\tau\in\mathcal{A}(n_{I_{\ell+1}+j}^{\ell+1})$, its leaf index is reassigned as $\Psi(\tau)=n_{I_{\ell+1}+j}^{\ell+1}$. Cells in the parent subarea that are not assigned to any child subarea no longer belong to a current leaf node and are assigned $\Psi(\tau)=\emptyset$. 

Fig.~\ref{fig:tree_example_part2} shows an example of tree expansion. The residual region associated with the root node $n_1^1$ is separated into three disconnected components due to newly discovered obstacles and covered cells. Accordingly, $n_1^1$ is expanded into three child nodes $n_1^2$, $n_2^2$, and $n_3^2$. The parent node $n_1^1$ becomes an internal node, while its three child nodes become the new leaves of the tree. As coverage proceeds, Fig.~\ref{fig:tree_example_part3} shows that $n_2^2$ is further expanded into child nodes $n_1^3$ and $n_2^3$. The updated hierarchical coverage tree is subsequently used to incrementally update the global tour, as described in the next subsection.

\subsection{Incremental Global Tour Planning}
\label{global_tour}

As the environment is incrementally revealed and covered, only a small portion of the hierarchical coverage tree is typically updated at each planning step, while most remaining regions stay unchanged. A complete recomputation of the global tour, however, may drastically alter the visiting order of these unchanged regions when several candidate orders have similar travel costs. This can lead to indecisive coverage behavior, unnecessary backtracking, and reduced execution consistency. To address these issues, the proposed algorithm retains the visiting order of unaffected regions in the global tour, while locally updating the affected portion.

Let \(\mathcal{N}_{\mathrm{active}}=\{n_i^{\ell}\in\mathcal{N}_{\mathrm{leaf}}:\mathcal{Q}(n_i^{\ell})\neq C\}\) be the set of active leaf nodes whose coverage has not yet been completed. Let the global tour before the current update be represented as an ordered sequence $\gamma^{-}_{\mathrm{global}}=[v_m]_{m=1}^{M}$, where $v_m$ denotes the $m$-th node in the tour. Each $v_m$ is examined sequentially with respect to the updated hierarchical coverage tree. If $v_m$ remains an active leaf node, its relative position in the tour is retained. If its coverage has been completed, it is removed from the tour. If $v_m$ has been expanded and becomes an internal node, it is replaced by an optimized visiting sequence of its newly generated child nodes. Specifically, consider the $m$-th tour node $v_m=n_i^{\ell}$ that has been expanded into $J$ child nodes $\{n_{i_1}^{\ell+1},\ldots,n_{i_J}^{\ell+1}\}$. Let $n_a$ and $n_b$ denote the predecessor and successor of $v_m$ in the current global tour, respectively, if they exist. A local tour starting from $n_a$, visiting all $J$ child nodes, and ending at $n_b$ is formulated as a Traveling Salesman Problem (TSP) with fixed start and end conditions. 

A complete weighted graph $\hat{\mathcal{G}}=(\hat{\mathcal{N}}, \hat{\mathcal{E}})$ is constructed for the TSP, where $\hat{\mathcal{N}} = \{n_a,n_{i_1}^{\ell+1},\ldots,n_{i_J}^{\ell+1},n_b\}$, and $\hat{\mathcal{E}}$ contains an edge between every pair of distinct nodes in $\hat{\mathcal{N}}$. Let $\mathcal{T}_{\mathrm{tr}} = \{\tau_j\in\mathcal{T}:\Phi(\tau_j)\neq O\}$ denote the cells not currently identified as obstacles. The weight of each edge is defined as the shortest-path distance between the centers of its incident nodes over $\mathcal{T}_{\mathrm{tr}}$, computed using the A* algorithm~\cite{hart1968formal}. The 2-opt method~\cite{aarts2003} is then employed to obtain an approximate solution $\hat{\pi}$ for the TSP. The tour node $v_m$ is replaced by the resulting child-node sequence, yielding $\gamma^{+}_{\mathrm{global}}=[v_1,\ldots,v_{m-1},n^{\ell+1}_{i_{\hat{\pi}(1)}},\ldots,n^{\ell+1}_{i_{\hat{\pi}(J)}},v_{m+1},\ldots,v_M]$. The same procedure is applied to all other expanded nodes in the global tour. In this way, the visiting order outside the locally refined portions is preserved, while the newly generated child nodes are locally arranged to reduce the travel cost. The first element of $\gamma^{+}_{\mathrm{global}}$ is selected as the target node $n_{\mathrm{target}}$ for local coverage planning. An example of global tour and corresponding target node is shown in Figs.~\ref{fig:tree_example_part2} and~\ref{fig:tree_example_part3}.

\RestyleAlgo{ruled}
\LinesNumbered
\begin{algorithm}[t]
\small
\caption{Tree Update (Secs.~\ref{node_update} and~\ref{tree_expand})}
\label{alg:tree_update}
\KwIn{Tree $\mathcal{G}$, symbolic map of tiling $\Phi$, changed cell set $\Delta\mathcal{T}$, leaf index map $\Psi$}
\KwOut{Updated tree $\mathcal{G}$ and leaf index map $\Psi$}

$\Delta\mathcal{N}_{\mathrm{leaf}} \leftarrow \emptyset$\;

\ForEach{$\tau_j \in \Delta\mathcal{T}$}
{
    \If{$\Psi(\tau_j) \neq \emptyset$}
    {
        $\Delta\mathcal{N}_{\mathrm{leaf}}
        \leftarrow
        \Delta\mathcal{N}_{\mathrm{leaf}}
        \cup \{\Psi(\tau_j)\}$\;
    }
}

\ForEach{$n_i^{\ell} \in \Delta\mathcal{N}_{\mathrm{leaf}}$}
{
    $\{\mathcal{R}(n_i^{\ell}),p(n_i^{\ell}),\mathcal{Q}(n_i^{\ell})\}\leftarrow \texttt{updateNodeInfo}(n_i^{\ell},\Phi)$\;
                
    \If{$\mathcal{Q}(n_i^{\ell}) \neq C$}
    {
        $\{\mathcal{S}_j\}_{j=1}^{J}
        \leftarrow
        \texttt{searchDisjointPortion}(\mathcal{R}(n_i^{\ell}))$\;

        \If{$J>1$}
        {
            \ForEach{$\tau_j \in \mathcal{A}(n_i^{\ell})$}
            {
                $\Psi(\tau_j) \leftarrow \emptyset$\;
            }
            $I_{\ell+1}
                \leftarrow
                |\mathcal{N}^{\ell+1}|$\;
            \For{$j \leftarrow 1$ \KwTo $J$}
            {
                $n_{I_{\ell+1}+j}^{\ell+1}
                \leftarrow
                \texttt{createChildNode}(\mathcal{S}_j,\ell+1)$\;
                
                $\mathcal{G}.\texttt{addNode}(n_{I_{\ell+1}+j}^{\ell+1})$\;
                $\mathcal{G}.\texttt{addEdge}(n_i^{\ell},n_{I_{\ell+1}+j}^{\ell+1})$\;

                \ForEach{$\tau \in \mathcal{S}_j$}
                {
                    $\Psi(\tau) \leftarrow n_{I_{\ell+1}+j}^{\ell+1}$\;
                }
            }
        }
    }
}
\end{algorithm}

\subsection{Local Coverage Path Planning}
\label{localCover}

Given the target node $n_{\mathrm{target}}$ selected from the updated global tour, the local planner progressively covers its residual region $\mathcal{R}(n_{\mathrm{target}})$. At each planning step, a target cell is selected to guide the robot's next coverage motion. The target-cell selection strategy is described below.

\begin{rem}
If $\mathcal{R}(n_{\mathrm{target}})$ is not adjacent to the robot's current position, the robot first follows a collision-free path to reach the target region and then proceeds with local coverage.
\end{rem}

\subsubsection{Target Cell Selection Strategy}

Let $\tau_{\mathcal{R}}$ denote the robot's current cell. Let $\mathcal{H}(\tau_{\mathcal{R}})\subseteq\mathcal{T}$ be the local vicinity defined as the $3\times3$ tiling centered at $\tau_{\mathcal{R}}$, as shown in Fig.~\ref{fig:targat_cell_select}. Then, a set $\mathcal{V}(\tau_{\mathcal{R}})\subseteq\mathcal{H}(\tau_{\mathcal{R}})$ is constructed to contain all candidate cells for target selection.

\begin{defn}[Candidate Cell]
A cell $\tau \in \mathcal{H}(\tau_{\mathcal{R}})$ is considered as a candidate if: i) $\tau \in \mathcal{R}(n_{\mathrm{target}})$, ii) $\Phi(\tau)=F^u$, and iii) the straight-line path from $\tau_{\mathcal{R}}$ to $\tau$ is collision-free.
\end{defn}

\begin{defn}[Lap]
A lap is a collision-free straight line passing through a sequence of cells and oriented perpendicular to the sweep direction of the coverage path.
\end{defn}

Next, target cell is selected within $\mathcal{V}(\tau_{\mathcal{R}})$ according to the priority of the left lap, current lap, and right lap. These directions are defined with respect to a fixed sweep direction. When multiple candidates exist on the same lap, the nearest one is selected. This strategy generates a systematic back-and-forth coverage pattern as follows. At any cell, the robot first attempts to move toward a target cell on the left lap and continues covering along the selected lap. When no eligible cell remains on the left lap, the robot proceeds along the current lap and subsequently shifts to the right lap. In this way, the robot performs the back-and-forth coverage lap by lap while shifting laps from left to right as they are covered.

\begin{figure}[t]
    \centering
    \subfloat[Local target cell selection.]{
        \includegraphics[width=0.48\columnwidth]{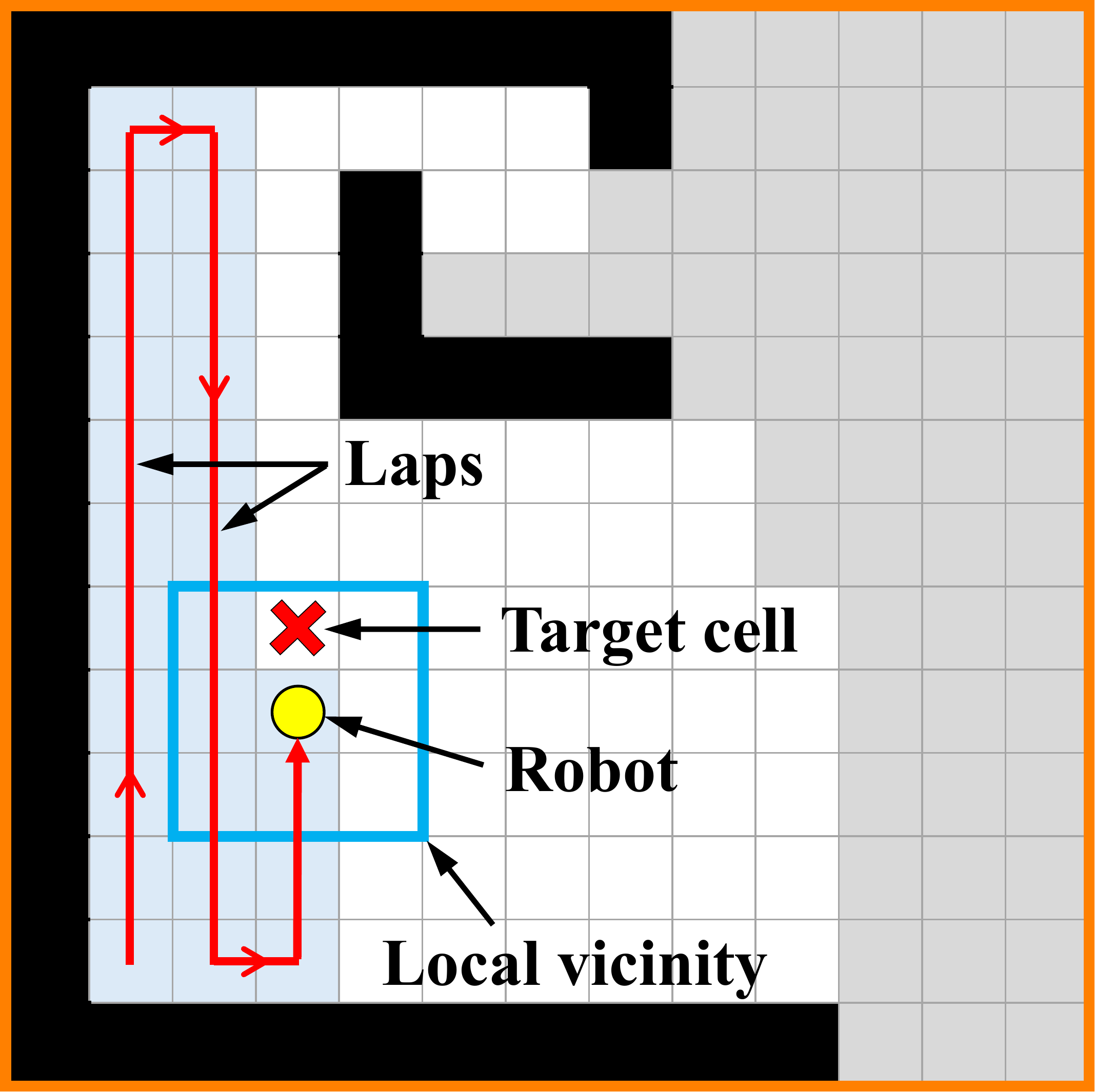}\label{fig:targat_cell_select}}\quad \hspace{-10pt}
    \centering
    \subfloat[Dead-end escape.]{
        \includegraphics[width=0.48\columnwidth]{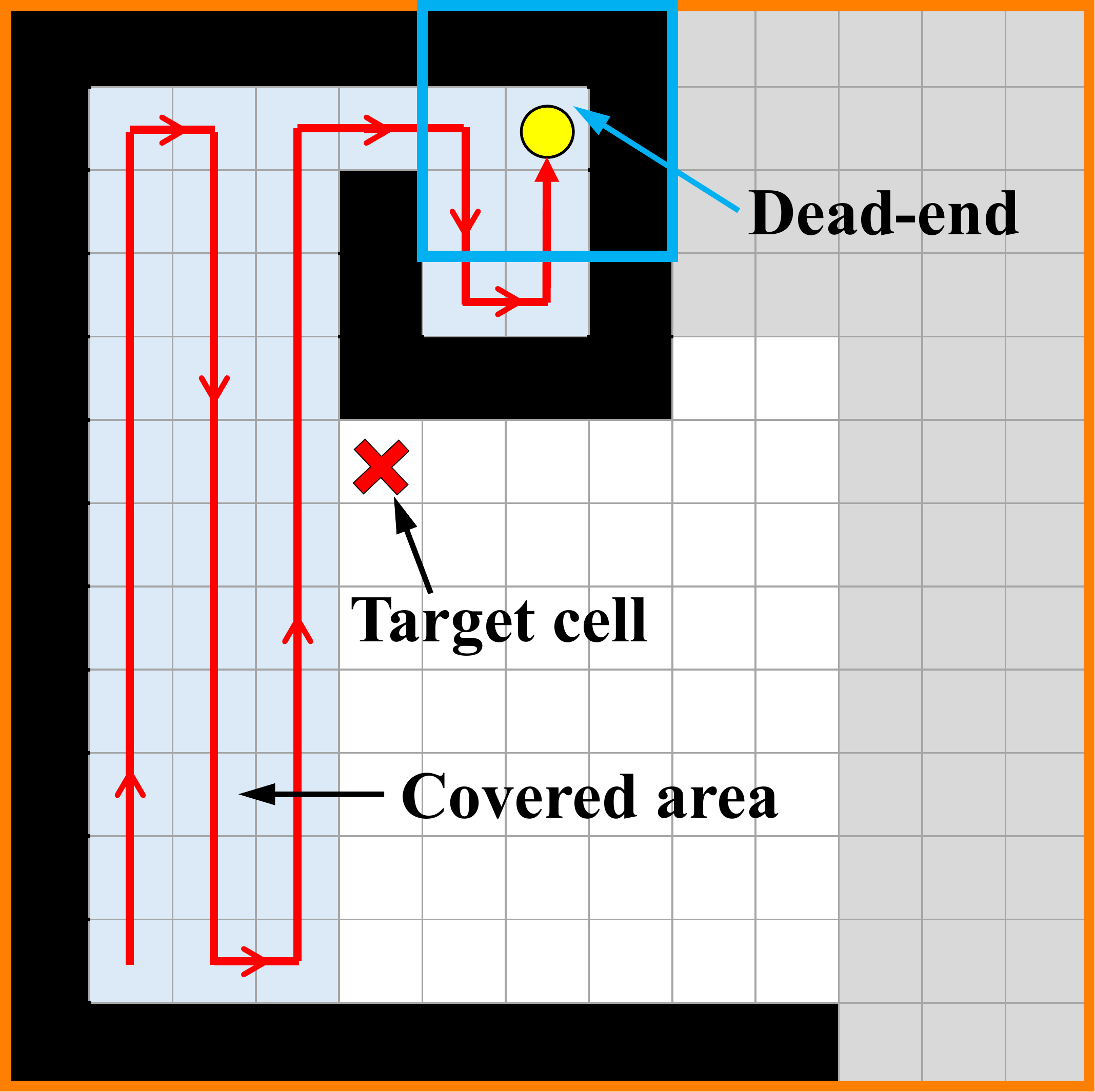}\label{fig:dead_end_escape}}\\
    \caption{Illustration of local target-cell selection and dead-end recovery. (a) The target cell is selected within the local vicinity according to the priority of neighboring laps, generating a back-and-forth coverage pattern. (b) When no candidate target cell is available in the local vicinity, the robot reaches a dead-end. The nearest free-uncovered cell adjacent to the covered area is then selected as the target, allowing the robot to escape the dead-end and resume coverage.}\label{fig:localCover} 
    \vspace{-1.0em}
 \end{figure}

\subsubsection{State Update Strategy}
Once the target cell is selected, the state of the current cell $\tau_{\mathcal R}$ is updated to $\Phi(\tau_{\mathcal R})=F^c$ unless both cells immediately above and below $\tau_{\mathcal R}$ on the current lap remain free and uncovered. In the latter case, $\Phi(\tau_{\mathcal R})$ remains $F^u$ to avoid premature interruption of the lap and preserve the structured back-and-forth coverage pattern.

Following the cell-state update, newly acquired sensor measurements and coverage progress are incorporated into the symbolic map $\Phi$. The residual region $\mathcal R(n_{\mathrm{target}})$ and the node state $\mathcal Q(n_{\mathrm{target}})$ are then updated accordingly. If the node becomes explored, i.e., $\mathcal{Q}(n_{\mathrm{target}})=E$, its remaining residual region contains only free-uncovered cells and the planner switches to the global-tour-aware coverage strategy described below. If $\mathcal{Q}(n_{\mathrm{target}})=C$, coverage of the target node is completed. In addition, if the updated residual region becomes disconnected, the target node is expanded; the global tour is then updated and a new target node is selected.

\subsubsection{Dead-End Recovery Strategy}
During navigation it is possible that the robot reaches a cell where it is unable to find the target cell because it is surrounded by only obstacles and covered regions, as shown in Fig.~\ref{fig:dead_end_escape}. This is known as a dead-end situation defined below.

\begin{defn}[Dead-end]
A dead-end occurs at the robot's current cell $\tau_{\mathcal{R}}$ if $\mathcal{V}(\tau_{\mathcal{R}})=\emptyset$ while $\mathcal{Q}(n_{\mathrm{target}})\neq C$, i.e., no candidate target cell is available although coverage of the target node remains incomplete.
\end{defn}

When a dead-end occurs, directly selecting an arbitrary uncovered cell within $\mathcal R(n_{\mathrm{target}})$ may introduce unnecessary traversal and disrupt the structured sweeping process. The robot therefore searches for the nearest free-uncovered cell adjacent to the covered area, as shown in Fig.~\ref{fig:dead_end_escape}. Then it follows the shortest path to this cell and resumes the target cell selection strategy from that location.

\subsubsection{Global-Tour-Aware Coverage Strategy}

When the target node becomes explored, i.e., $\mathcal{Q}(n_{\mathrm{target}})=E$, all cells in $\mathcal{R}(n_{\mathrm{target}})$ are known to be free and uncovered. A local coverage path is generated by formulating a TSP that starts from the robot's current cell $\tau_{\mathcal R}$ and visits every cell in $\mathcal R(n_{\mathrm{target}})$. To maintain consistency with the global tour, let $n_{\mathrm{next}}$ denote the node following $n_{\mathrm{target}}$, if it exists. The center $p(n_{\mathrm{next}})$ is introduced as a virtual terminal point of the local coverage path. This encourages the path to terminate at a location favorable for transitioning to the next target node, thereby reducing transition cost. If $n_{\mathrm{next}}$ does not exist, the terminal location is left unconstrained and the corresponding open TSP is solved with a fixed start and free end. After the resulting coverage path is executed, all cells in $\mathcal R(n_{\mathrm{target}})$ become covered and the node state is updated to $\mathcal{Q}(n_{\mathrm{target}})=C$.

\section{Algorithm Analysis}
\label{sec:analysis}

This section presents a formal analysis of both computational complexity and theoretical guarantees. 

\begin{lem}
    Hierarchical coverage tree update has a time complexity of $O(|\mathcal{T}|)$. 
    \label{lem:tree_update}
\end{lem}

\begin{proof}  
As described in Alg.~\ref{alg:tree_update}, the affected leaf-node set $\Delta\mathcal{N}_{\mathrm{leaf}}$ is first identified by mapping each changed cell $\tau_j\in\Delta\mathcal{T}$ to its corresponding leaf node through the leaf index map $\Psi$. Since each lookup of $\Psi(\tau_j)$ takes constant time, this step requires $O(|\Delta\mathcal{T}|)$. Consider an affected leaf node $n_i^\ell\in\Delta\mathcal{N}_{\mathrm{leaf}}$. Updating the node information requires at most linear time in the size of its residual region, with complexity $O(|\mathcal{R}(n_i^{\ell})|)$. The subsequent flood-fill search visits each cell in $\mathcal{R}(n_i^\ell)$ at most once and therefore requires $O(|\mathcal{R}(n_i^{\ell})|)$. If the residual region remains path-connected, no modification of the tree is required. If it is divided into multiple disjoint portions, new child nodes are created and the leaf index map $\Psi$ is updated. This operation requires at most one traversal of the associated subarea and the identified components, giving a complexity of $O(|\mathcal{A}(n_i^{\ell})|+|\mathcal{R}(n_i^{\ell})|)$. Therefore, the total complexity of the tree update is $O(|\Delta\mathcal{T}|+\sum_{n_i^\ell\in\Delta\mathcal{N}_{\mathrm{leaf}}}(|\mathcal{A}(n_i^{\ell})|+|\mathcal{R}(n_i^{\ell})|))$. Since $\mathcal{R}(n_i^{\ell}) \subseteq \mathcal{A}(n_i^{\ell})$, this simplifies to $O(|\Delta\mathcal{T}|+\sum_{n_i^\ell\in\Delta\mathcal{N}_{\mathrm{leaf}}}|\mathcal{A}(n_i^{\ell})|)$. Moreover, the associated subareas of distinct leaf nodes are mutually disjoint, such that $\sum_{n_i^\ell\in\Delta\mathcal{N}_{\mathrm{leaf}}}|\mathcal{A}(n_i^{\ell})| \leq |\mathcal{T}|$. Since $\Delta \mathcal{T} \subseteq \mathcal{T}$, the overall complexity of the tree update is $O(\mathcal{T})$.
\end{proof}

\begin{lem}
The incremental global tour planning has a time complexity of
$O(M+\sum_{h=1}^{H}J_h^2)$.
\label{lem:global_tour}
\end{lem}

\begin{proof}
Let $M$ denote the number of nodes in the current global tour, $H$ the number of expanded tour nodes, and $J_h$ the number of child nodes generated by the $h$-th expanded node. Examining all $M$ tour nodes against the updated tree requires $O(M)$ time. For the $h$-th expanded node, the corresponding local TSP contains at most $J_h+2$ nodes, including its $J_h$ child nodes and the predecessor and successor, when available. The 2-opt method has complexity of $O((J_h+2)^2)=O(J_h^2)$~\cite{shen2022ct}. Therefore, updating all expanded nodes requires $O(\sum_{h=1}^{H}J_h^2)$ time, yielding an overall complexity of $O(M+\sum_{h=1}^{H}J_h^2)$.
\end{proof}

\begin{lem}
The local coverage path planning has a time complexity of
$O(|\mathcal{R}(n_{\mathrm{target}})|^2)$.
\label{lem:localCover}
\end{lem}

\begin{proof}
By default, target cell is selected within a $3\times3$ local vicinity. Since at most eight neighboring cells are examined, this operation requires $O(1)$ time. When a dead-end occurs, searching for the target cell requires at most $O(|\mathcal R(n_{\mathrm{target}})|)$ time. For an \textit{explored} target node, the remaining coverage problem is formulated as a TSP over at most $|\mathcal R(n_{\mathrm{target}})|$ cells. The 2-opt method therefore requires $O(|\mathcal R(n_{\mathrm{target}})|^2)$ time~\cite{shen2022ct}. Since the TSP-based planning dominates the other steps, the overall time complexity is $O(|\mathcal R(n_{\mathrm{target}})|^2)$.
\end{proof}

\begin{thm}
The proposed algorithm has a time complexity of
$O(|\mathcal{T}|^2)$ for each planning and update cycle.
\label{thm:overall_complexity}
\end{thm}

\begin{proof}
From Lemmas~\ref{lem:tree_update}-\ref{lem:localCover}, the proposed algorithm has a complexity of
$O\left(|\mathcal{T}|+M+\sum_{h=1}^{H}J_h^2+|\mathcal{R}(n_{\mathrm{target}})|^2\right)$. Since $M\leq|\mathcal{T}|$ and $|\mathcal{R}(n_{\mathrm{target}})|\leq|\mathcal{T}|$, the complexity simplifies to $O(|\mathcal{T}|+\sum_{h=1}^{H}J_h^2)$. Moreover, the child nodes generated from different expanded leaf nodes correspond to mutually disjoint portions of the tiling, yielding
$\sum_{h=1}^{H}J_h\leq|\mathcal{T}|$ and thus
$\sum_{h=1}^{H}J_h^2\leq\left(\sum_{h=1}^{H}J_h\right)^2
\leq|\mathcal{T}|^2$. Therefore, the time complexity of the proposed algorithm is $O(|\mathcal{T}|^2)$.
\end{proof}

\begin{lem}
For a finite tiling $\mathcal T$, the hierarchical coverage tree undergoes at most $|\mathcal T|-1$ expansion events.
\label{lem:finite_expand}
\end{lem}

\begin{proof}
The tree is initialized with a single node. A tree expansion occurs only when the residual region of a leaf node is divided into $J>1$ disjoint components. The expanded node then becomes an internal node and is replaced by $J$ child leaf nodes. Therefore, each expansion increases the number of leaf nodes by $J-1\geq1$. Since the associated subareas of distinct leaf nodes are mutually disjoint and each leaf node contains at least one cell, the total number of leaf nodes is bounded by $|\mathcal T|$. Starting from a single leaf node, the number of tree-expansion events is therefore at most $|\mathcal T|-1$.
\end{proof}

\begin{thm}[Complete Coverage]
\label{thm:completeness}
The proposed algorithm achieves complete coverage of the obstacle-free space.
\end{thm}

\begin{proof}
According to Lemma~\ref{lem:finite_expand}, the hierarchical coverage tree undergoes only finitely many expansion events. Whenever a leaf node is expanded, its residual region is partitioned into disjoint child regions whose union equals the residual region of the parent. Hence, tree expansion does not discard any uncovered region. After each tree update, $\mathcal{N}_{\mathrm{active}}$ contains all leaf nodes whose coverage has not yet been completed, and these nodes are maintained in the global tour. During local coverage, the selected target node progressively explores or covers its residual region until it becomes explored, is completed, or is expanded; in the latter case, its residual region is inherited by its child nodes and remains in the global tour. Since $\mathcal{T}$ is finite and tree expansion can occur only finitely many times, repeated application of the above process eventually eliminates all active leaf nodes, i.e., $\mathcal{N}_{\mathrm{active}}=\emptyset$. Consequently, no obstacle-free cell remains uncovered, and complete coverage of the obstacle-free space is achieved.
\end{proof}

\begin{figure}[t]
    \centering
    \subfloat[Office.]{
        \includegraphics[width=0.48\columnwidth]{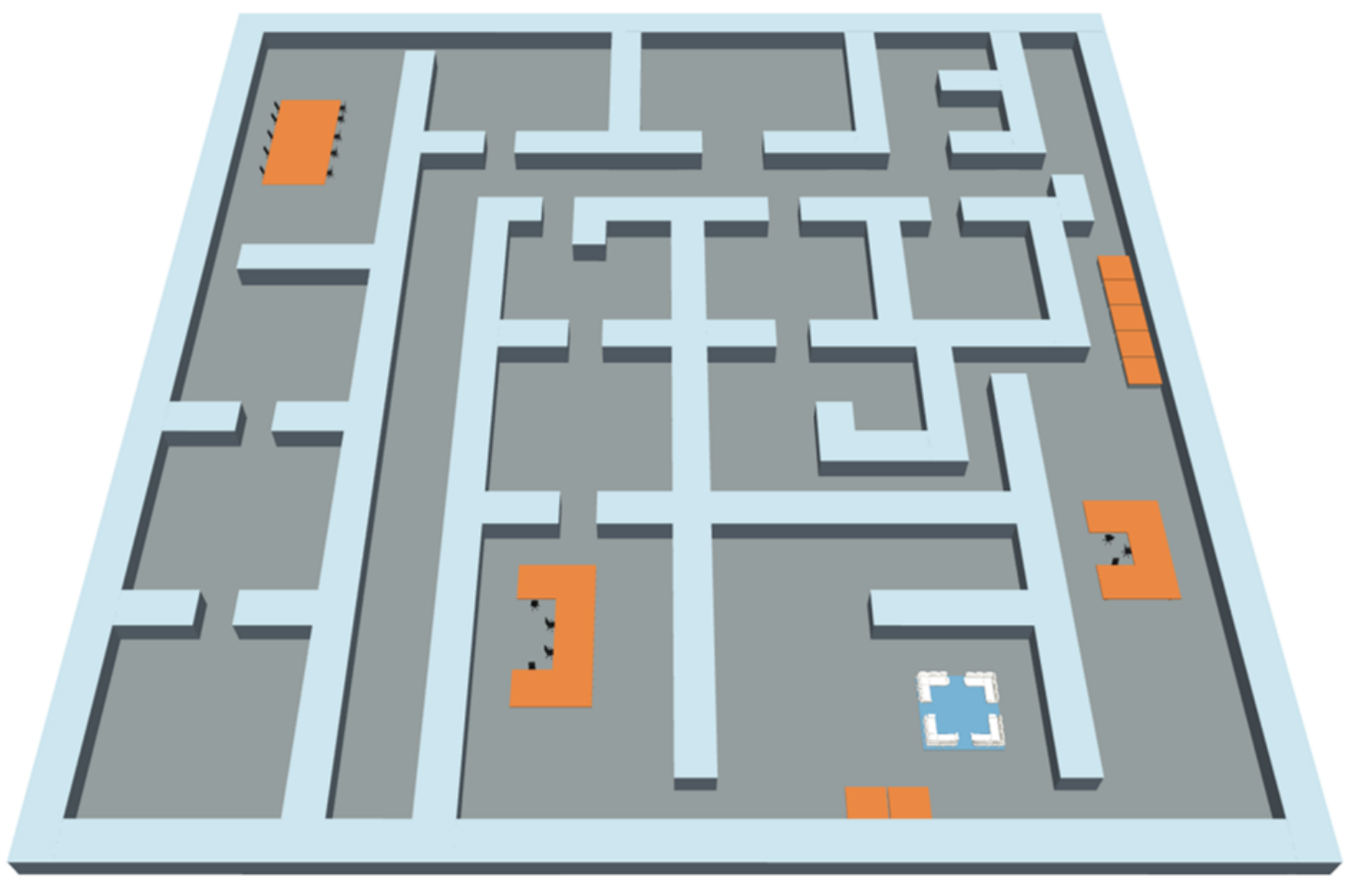}\label{fig:office}}\quad \hspace{-10pt}
    \centering
    \subfloat[Warehouse.]{
        \includegraphics[width=0.48\columnwidth]{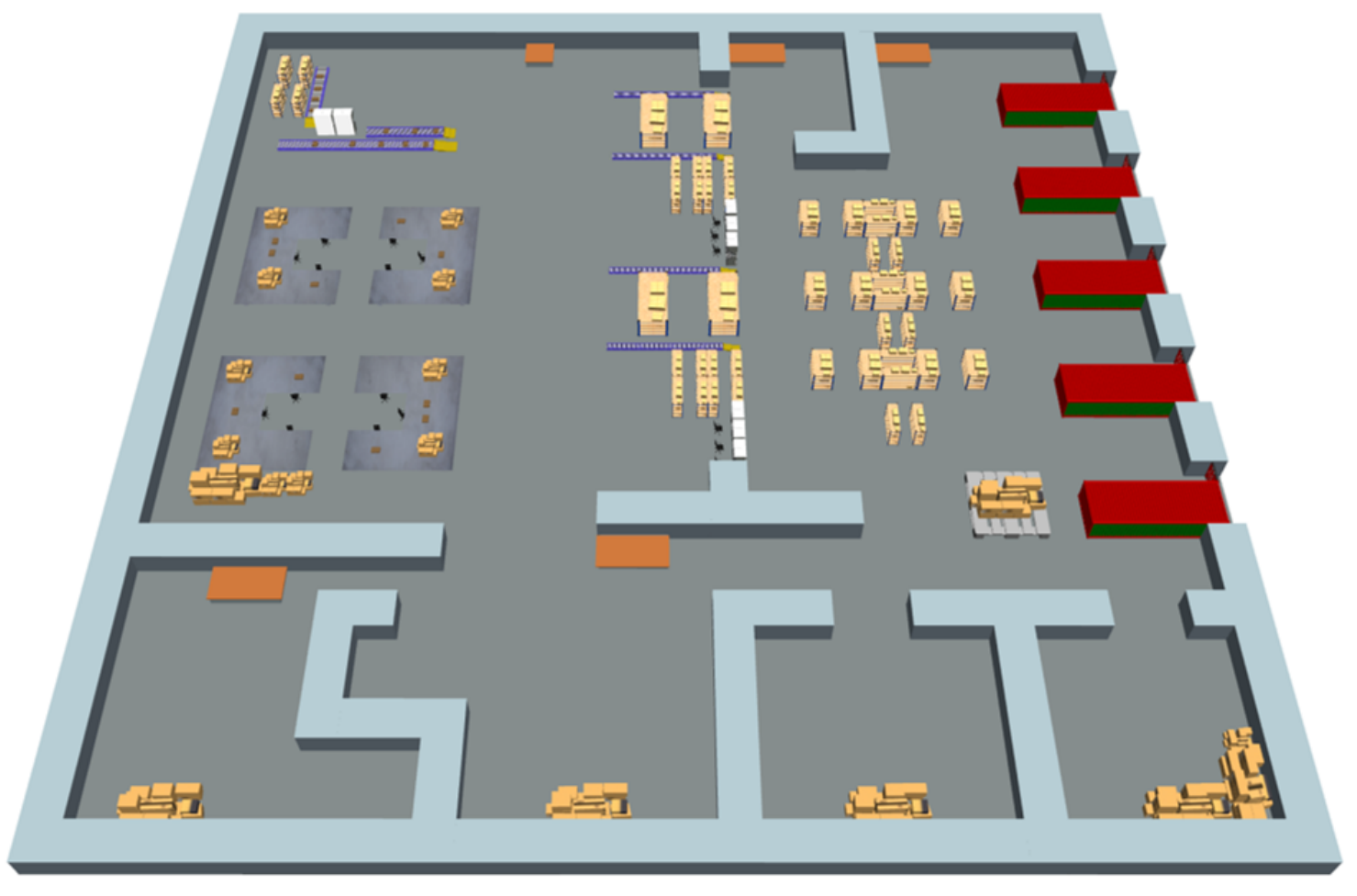}\label{fig:warehouse}}\vspace{0.5em}\\
        \centering
    \subfloat[Forest.]{
        \includegraphics[width=0.48\columnwidth]{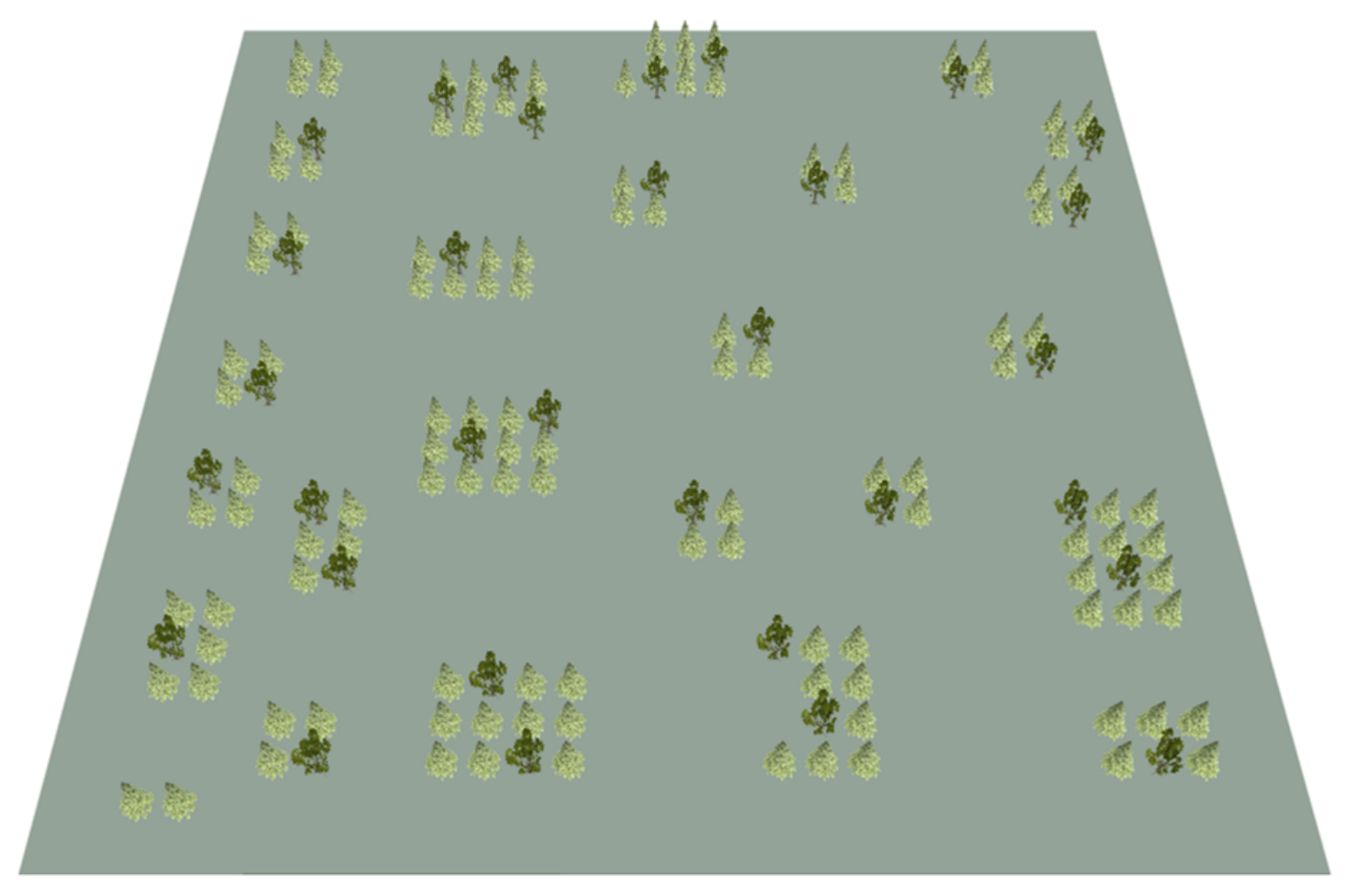}\label{fig:forest}}\quad \hspace{-10pt}
    \centering
    \subfloat[Mall 1.]{
        \includegraphics[width=0.48\columnwidth]{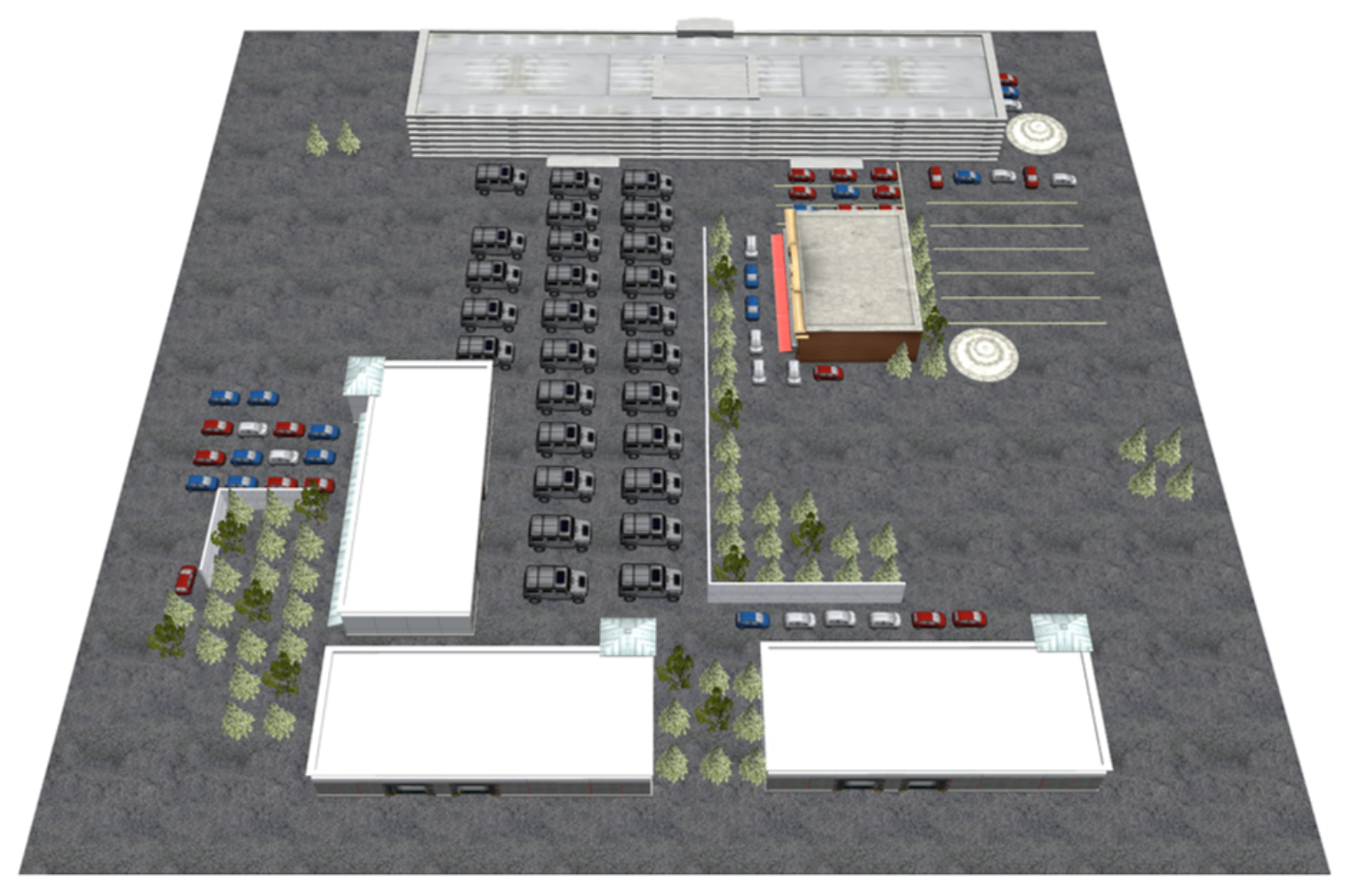}\label{fig:mall1}}\vspace{0.5em}\\
        \centering
    \subfloat[Mall 2.]{
        \includegraphics[width=0.48\columnwidth]{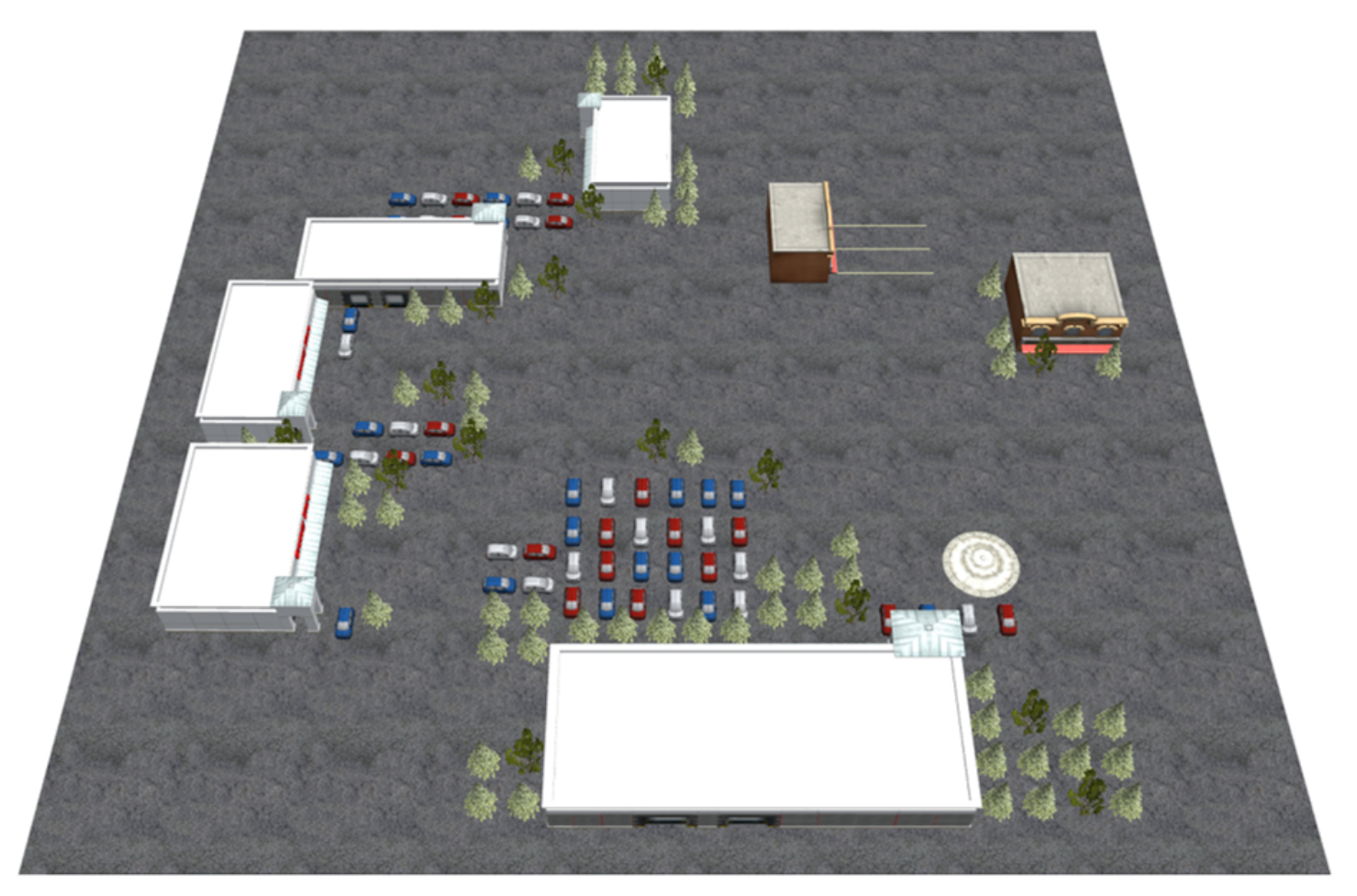}\label{fig:mall2}}\quad \hspace{-10pt}
    \centering
    \subfloat[Mall 3.]{
        \includegraphics[width=0.48\columnwidth]{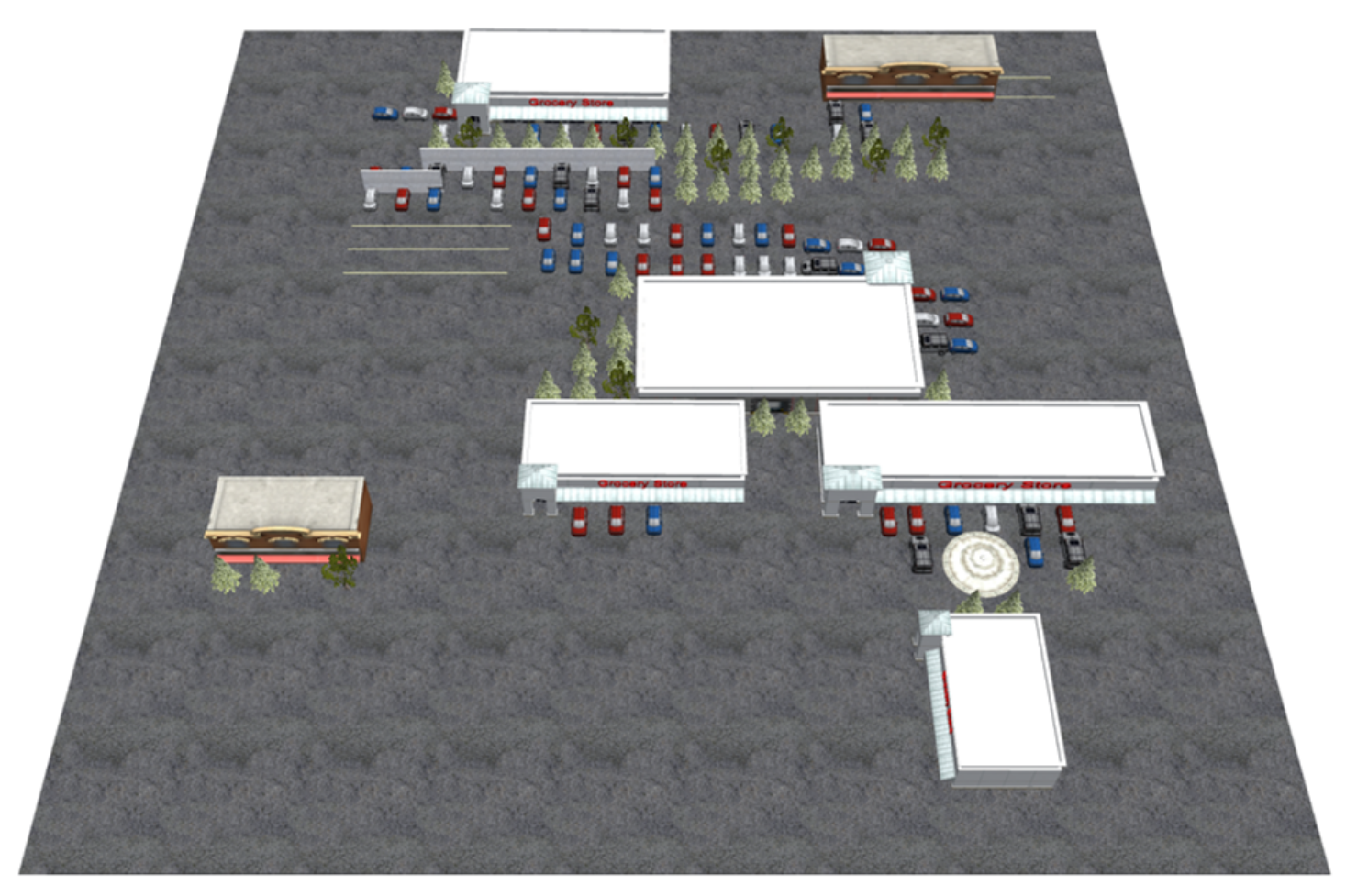}\label{fig:mall3}}\\
    \caption{An Overview of the test scenarios for simulation experiments.}\label{fig:simulation_scenario} 
    \vspace{-1.0em}
 \end{figure}

  \begin{figure*}[t]  \includegraphics[width=0.98\textwidth]{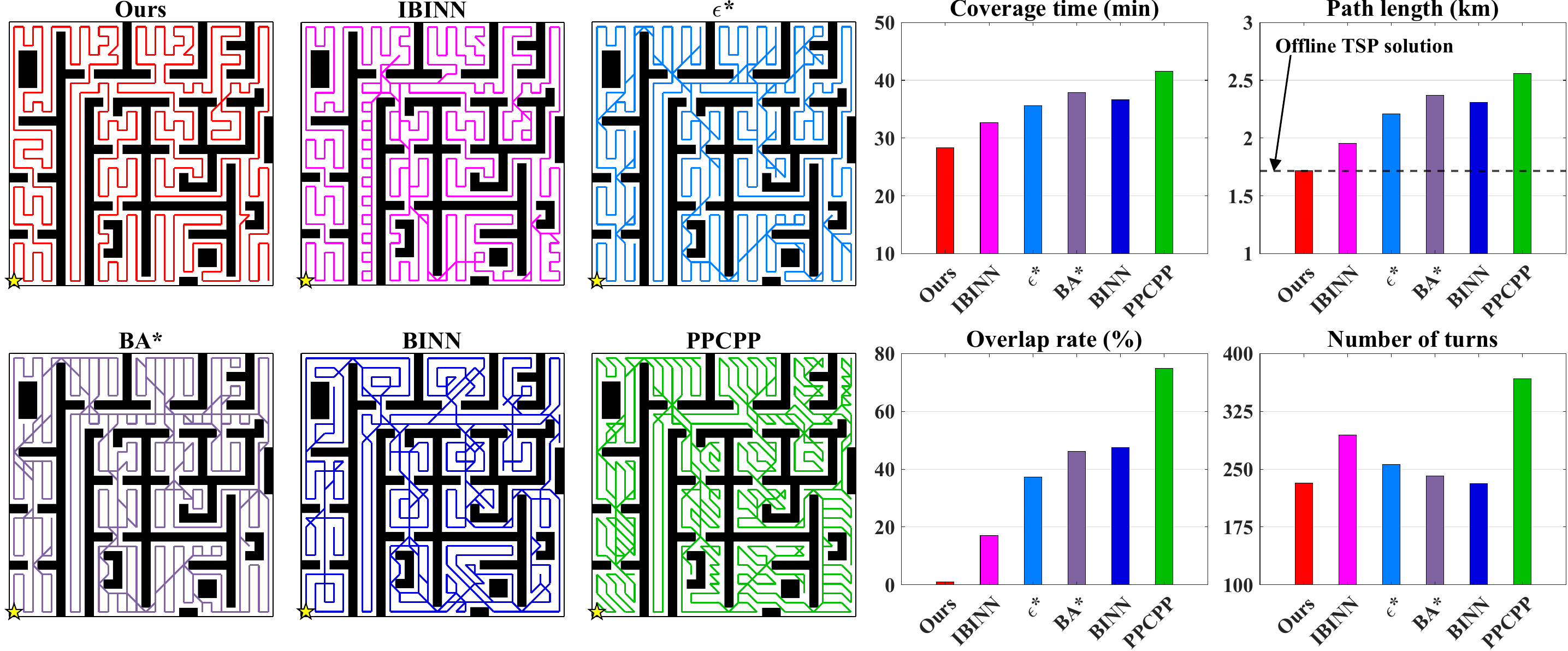}
    \caption{\textit{Office Scenario:} Performance comparison with the baseline algorithms.}
  \label{fig:simulation_res_office}
\end{figure*}

 \begin{figure*}[t]  \includegraphics[width=0.98\textwidth]{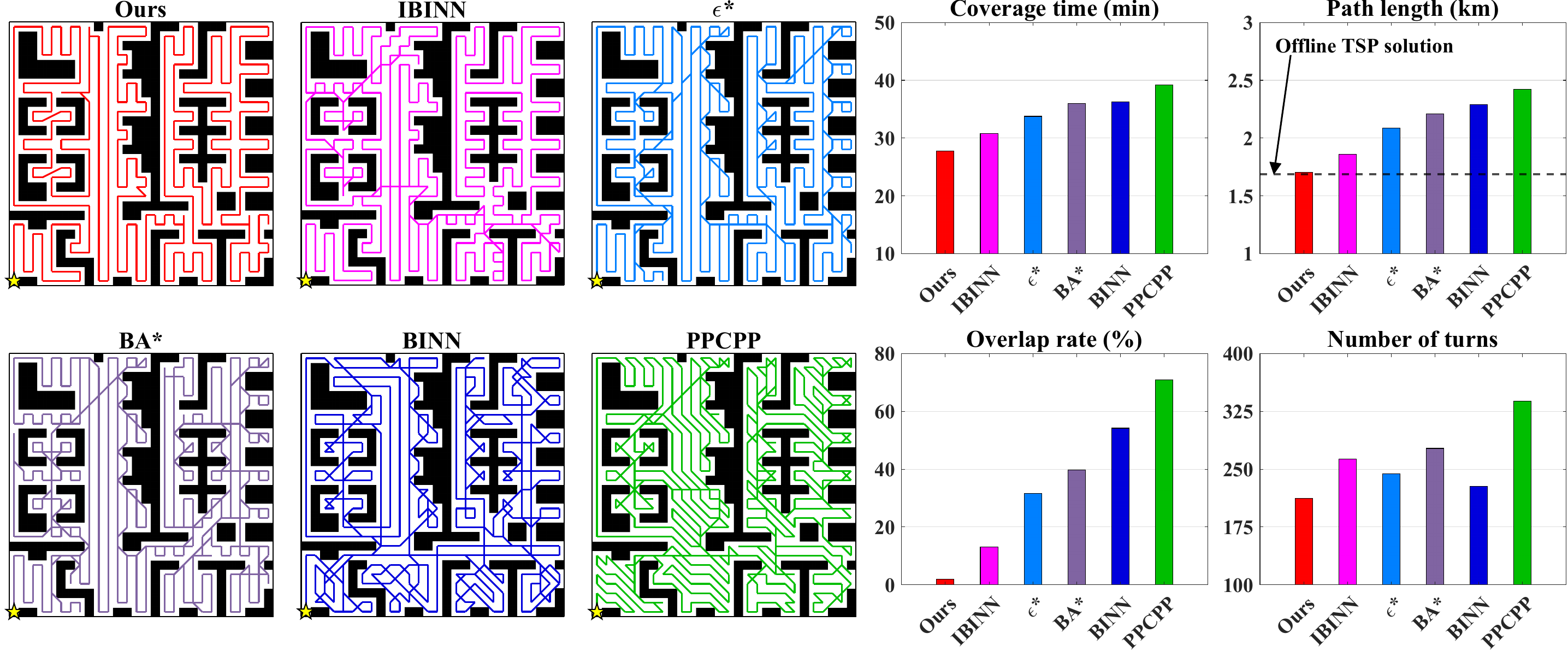}
    \caption{\textit{Warehouse Scenario:} Performance comparison with the baseline algorithms.}\label{fig:simulation_res_warehouse}
  \vspace{-1.0em}
\end{figure*}

 \begin{figure*}[t]  \includegraphics[width=0.98\textwidth]{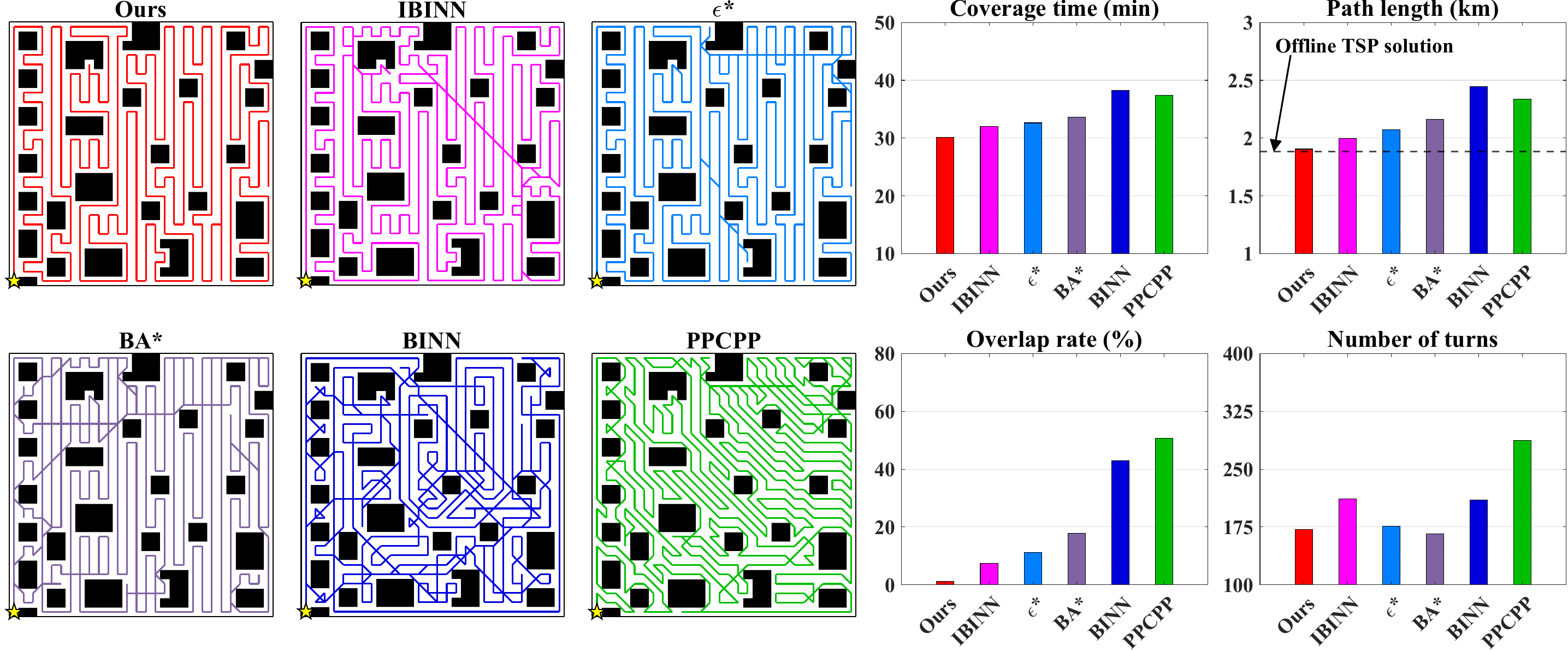}
    \caption{\textit{Forest Scenario:} Performance comparison with the baseline algorithms.}\label{fig:simulation_res_forest}
\end{figure*}

\begin{figure*}[t]  \includegraphics[width=0.98\textwidth]{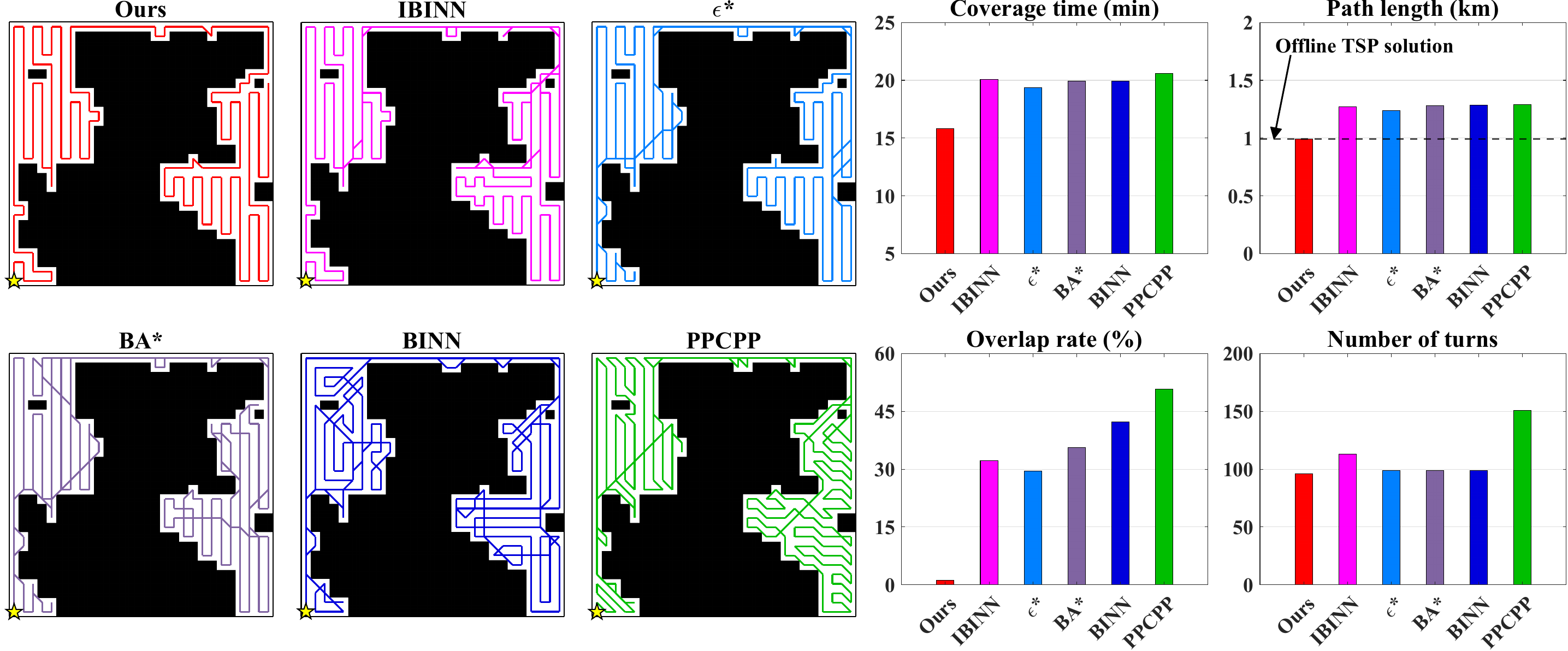}
    \caption{\textit{Mall 1 Scenario:} Performance comparison with the baseline algorithms.}\label{fig:simulation_res_mall_one}
  \vspace{-1.0em}
\end{figure*}

\begin{figure*}[t]  \includegraphics[width=0.98\textwidth]{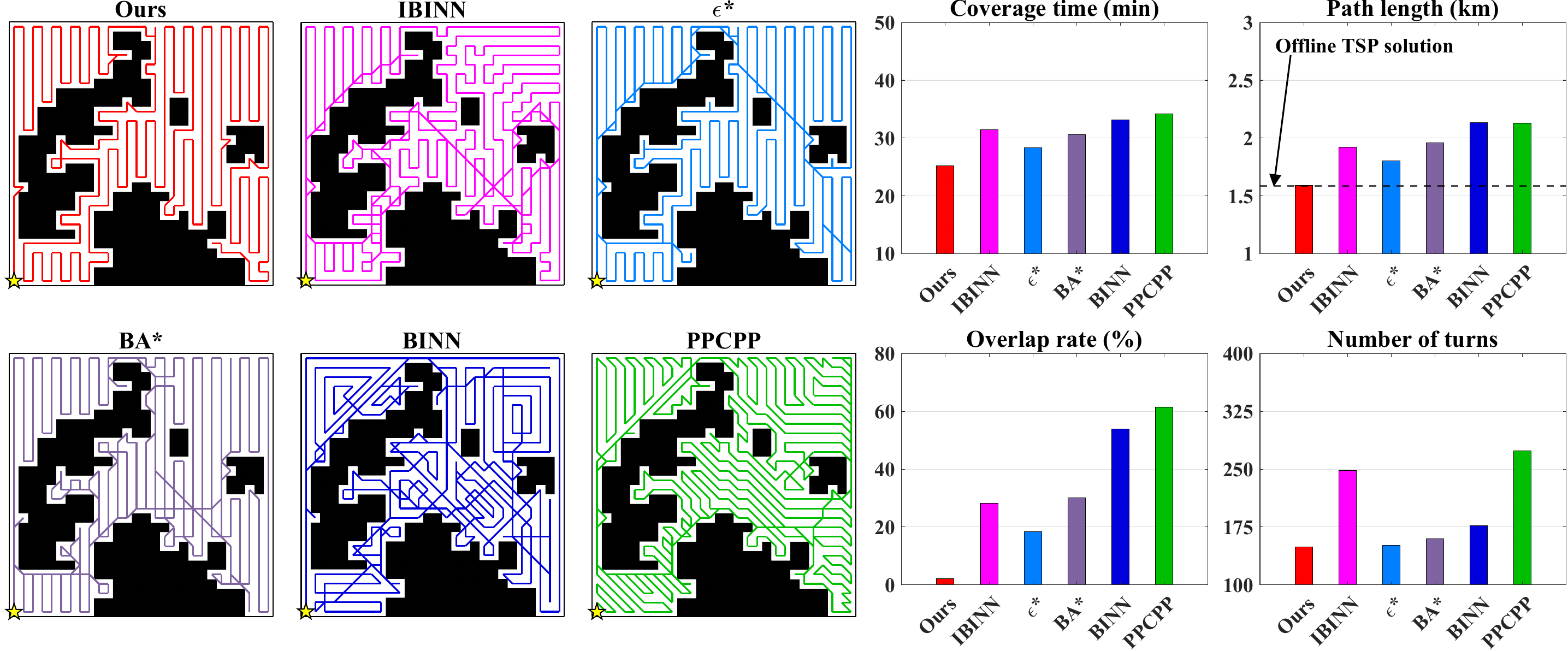}
    \caption{\textit{Mall 2 Scenario:} Performance comparison with the baseline algorithms.}\label{fig:simulation_res_mall_two}
\end{figure*}

\begin{figure*}[t]  \includegraphics[width=0.98\textwidth]{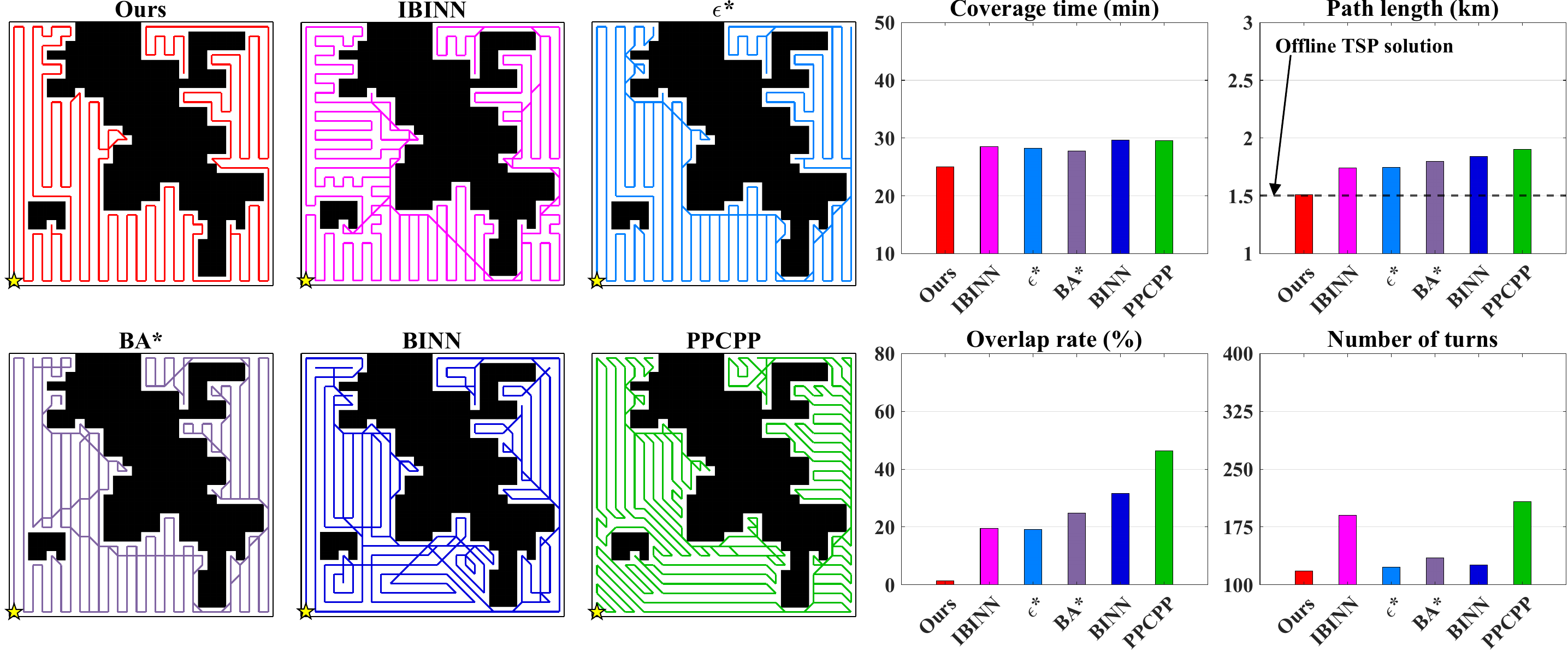}
    \caption{\textit{Mall 3 Scenario:} Performance comparison with the baseline algorithms.}\label{fig:simulation_res_mall_three}
  \vspace{-1.0em}
\end{figure*}

\begin{thm}[Approximation Ratio of Local Tour Refinement]
\label{thm:tour_approximation}
After $K$ tour updates induced by tree expansion, the global tour maintained by the proposed algorithm satisfies $\sigma_{\mathrm{inc}} =\frac{l_K}{l_K^{*}} \leq 1+\frac{2\sum_{k=1}^{K}(J_k+1)r_k}{l_K^{*}}$.
\end{thm}

\begin{proof}
Consider the $k$-th tree expansion. Let $l_k$ and $l_k^{*}$ denote the lengths of the updated tour and globally optimal tour, respectively. Suppose the expanded node $n_{p_k}$ generates $J_k$ child nodes $\mathcal{C}_k$. We define the refinement radius as $r_k=\max_{n\in\mathcal{C}_k}d(n_{p_k},n)$, which represents the maximum distance between the parent node and its child nodes. Since every child node is within distance $r_k$ of $n_{p_k}$, the triangle inequality implies that replacing the parent by any feasible ordering of its $J_k$ child nodes introduces at most $r_k$ at each boundary connection and at most $2r_k$ between two consecutive child nodes. Hence, the resulting increase in tour length is bounded by $r_k+2(J_k-1)r_k+r_k=2J_kr_k$, yielding $l_k\leq l_{k-1}+2J_kr_k$. If the expanded parent is an endpoint of the tour, one boundary connection is absent and the same upper bound still holds.

Next, consider the globally optimal tour. We construct a feasible tour for the pre-expansion node set from this optimal tour. Specifically, all but one child node in $\mathcal{C}_k$ are removed by directly connecting their predecessor and successor nodes. By the triangle inequality, these shortcut operations do not increase the tour length. The remaining child node is then replaced by its parent $n_{p_k}$. Since the distance between the parent and any child is at most $r_k$, this replacement increases the two incident connections by at most $2r_k$ in total. Therefore, the constructed pre-expansion tour has length no greater than $l_k^{*}+2r_k$. Let $l_{k-1}^{*}$ be the globally optimal tour length before expansion, we obtain $l_k^{*}\geq l_{k-1}^{*}-2r_k$. Let $\epsilon_k=l_k-l_k^{*}$ denote the suboptimality gap between the updated tour and globally optimal tour. Combining the above results gives $\epsilon_k\leq\epsilon_{k-1}+2(J_k+1)r_k$. The initial global tour contains only the root node and is therefore trivially optimal, giving $\epsilon_0=0$. Recursively applying this inequality over the $K$ expansion events yields $l_K-l_K^{*} \leq 2\sum_{k=1}^{K}(J_k+1)r_k$. Therefore, for $l_K^{*}>0$, the approximation ratio satisfies $\sigma_{\mathrm{inc}} =\frac{l_K}{l_K^{*}} \leq 1+\frac{2\sum_{k=1}^{K}(J_k+1)r_k}{l_K^{*}}$.

The bound indicates that the approximation ratio introduced by incremental tour refinement is determined by the number of newly generated child nodes and their spatial deviations from the corresponding parent nodes. Since tree expansion is local, the refinement radius $r_k$ is typically small, leading to a limited degradation from global optimality.
\end{proof}

\section{Results and Discussion}
\label{sec:results}

This section evaluates TRACE through high-fidelity simulations in Gazebo and real-robot experiments. The simulation studies compare TRACE with five baseline algorithms in terms of coverage time, path length, overlap ratio, and number of turns, and further examine its path quality and robustness to sensing and localization uncertainties.

\subsection{Validation on a Simulation Platform}

\subsubsection{Simulated Robot and Test Scenarios}

A mobile robot is simulated with motion constraints including a maximum translational velocity of $2\,\mathrm{m/s}$ and a minimum turning radius of $1\,\mathrm{m}$. It is equipped with a LiDAR providing a $360^\circ$ field of view and a maximum sensing range of $12\,\mathrm{m}$. Six $90\,\mathrm{m} \times 90\,\mathrm{m}$ with different obstacle arrangements and spatial structures are constructed for evaluation, as shown in Fig.~\ref{fig:simulation_scenario}. The scenarios include representations of an office, warehouse, forest, and mall and are designed to cover different levels of obstacle density, corridor structure, and spatial complexity. Each environment is represented by a $30 \times 30$ tiling for mapping and coverage planning. At the beginning of each run, the interior obstacle layout is unknown to the robot and is progressively revealed through onboard sensing during coverage.

\subsubsection{Baseline Algorithms} 

The proposed algorithm is compared with five online CPP algorithms: BINN~\cite{luo2008bioinspired}, BA$^*$~\cite{viet2013ba}, $\epsilon^*$~\cite{song2018}, PPCPP~\cite{hassan2019ppcpp}, and IBINN~\cite{huo2025}. All algorithms are implemented in C$++$ and evaluated using the same robot model, sensing configuration, map resolution, starting conditions, and computing platform. The experiments are conducted on a computer equipped with a $2.60$-GHz processor and $32$ GB of RAM. Algorithm-specific parameters follow the settings recommended in the corresponding publications.

\subsubsection{Performance Metrics} 

The following performance metrics are used to quantitatively compare the algorithms:

\begin{itemize}
    \item Coverage Time: elapsed time from the start of operation until all obstacle-free cells are covered. This metric jointly reflects travel distance and maneuvering costs under the robot motion constraints.
    \item Path Length: total distance traveled by the robot from initialization until complete coverage.
    \item Overlap Ratio: the amount of repeated coverage normalized by the total obstacle-free area.
    \item Number of Turns: the cumulative heading change along the executed trajectory, normalized by $90^\circ$.
\end{itemize}

\begin{figure}[t]
        \centering        \includegraphics[width=0.4\textwidth]{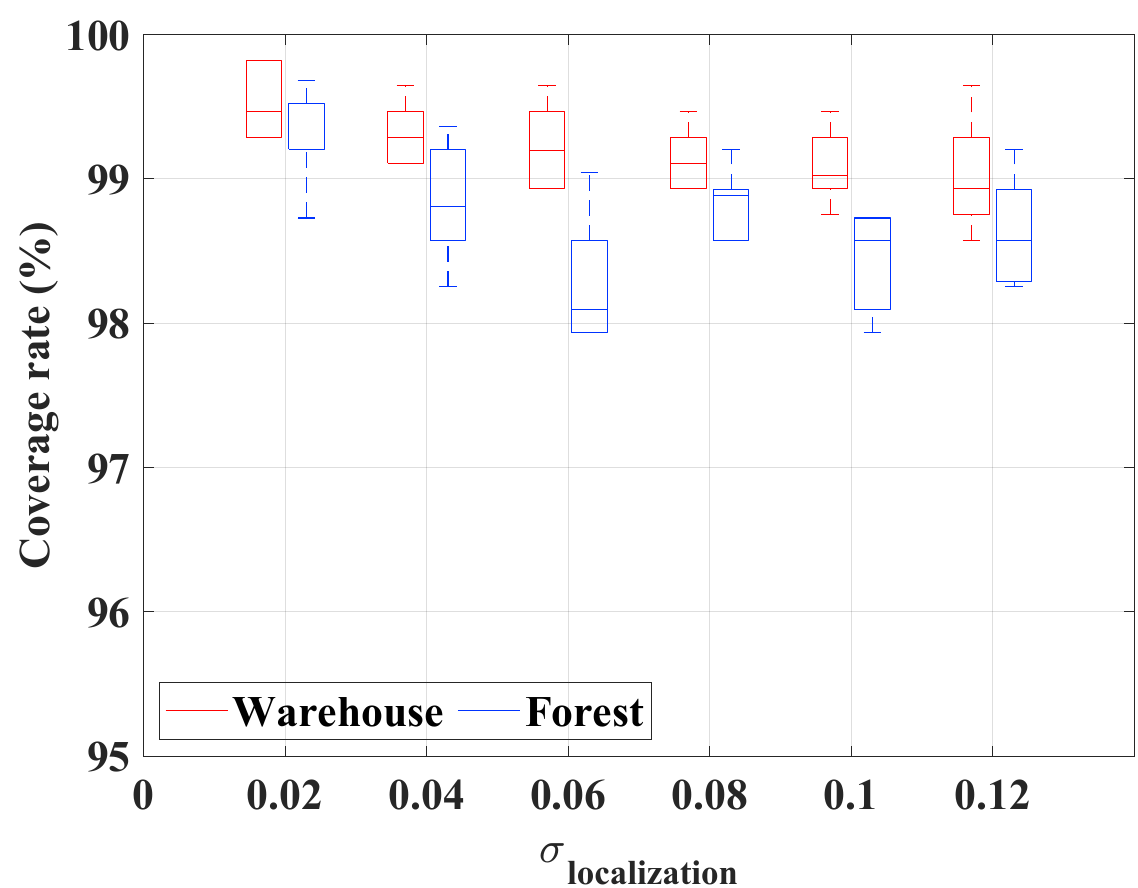}
    \caption{Coverage rate versus localization noise.}\label{fig:CoverageRatio_obsMap2_obsMap3} 
    \vspace{-1.0em}
 \end{figure}

 \begin{figure*}[t]  \includegraphics[width=0.98\textwidth]{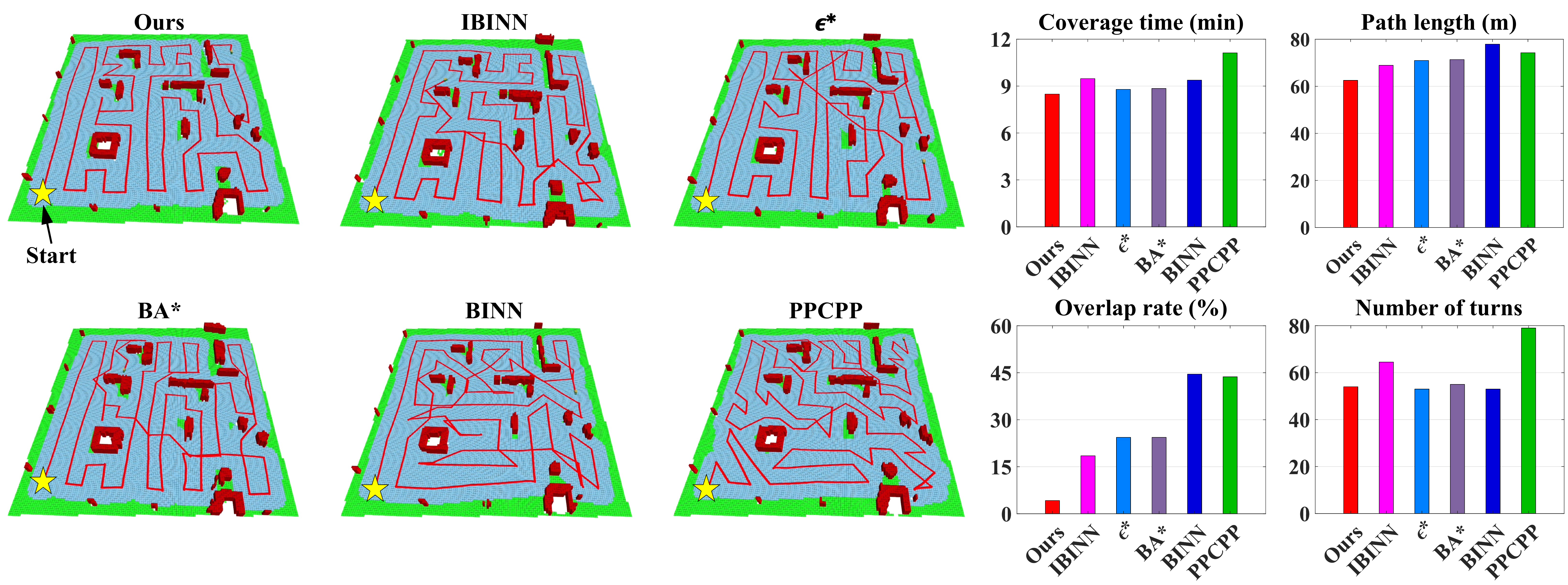}
    \caption{\textit{Real-Robot Experiment:} Performance comparison with the baseline algorithms in a laboratory with complex obstacle layout.}\label{fig:experiment_res}
  \vspace{-1.0em}
\end{figure*}

\subsubsection{Comparative Evaluation Results}

Figs.~\ref{fig:simulation_res_office}-\ref{fig:simulation_res_mall_three} shows the coverage paths generated by TRACE and the baseline methods in six scenarios. The robot starts from the bottom-left corner and initially has no knowledge of the interior obstacle layout. As coverage proceeds, each method updates its plan using the newly observed environment until all reachable free cells have been covered. As observed, TRACE produces efficient coverage paths with significantly less overlap across the six scenarios. These figures also present the quantitative results, where TRACE achieves favorable performance in coverage time, path length, overlap ratio, and number of turns. The average total computation time of TRACE per iteration, including hierarchical coverage tree construction, global tour planning, and local coverage path generation, is on the order of $10^{-2}\,\mathrm{s}$.

 \begin{figure}[t]
        \centering        \includegraphics[width=0.46\textwidth]{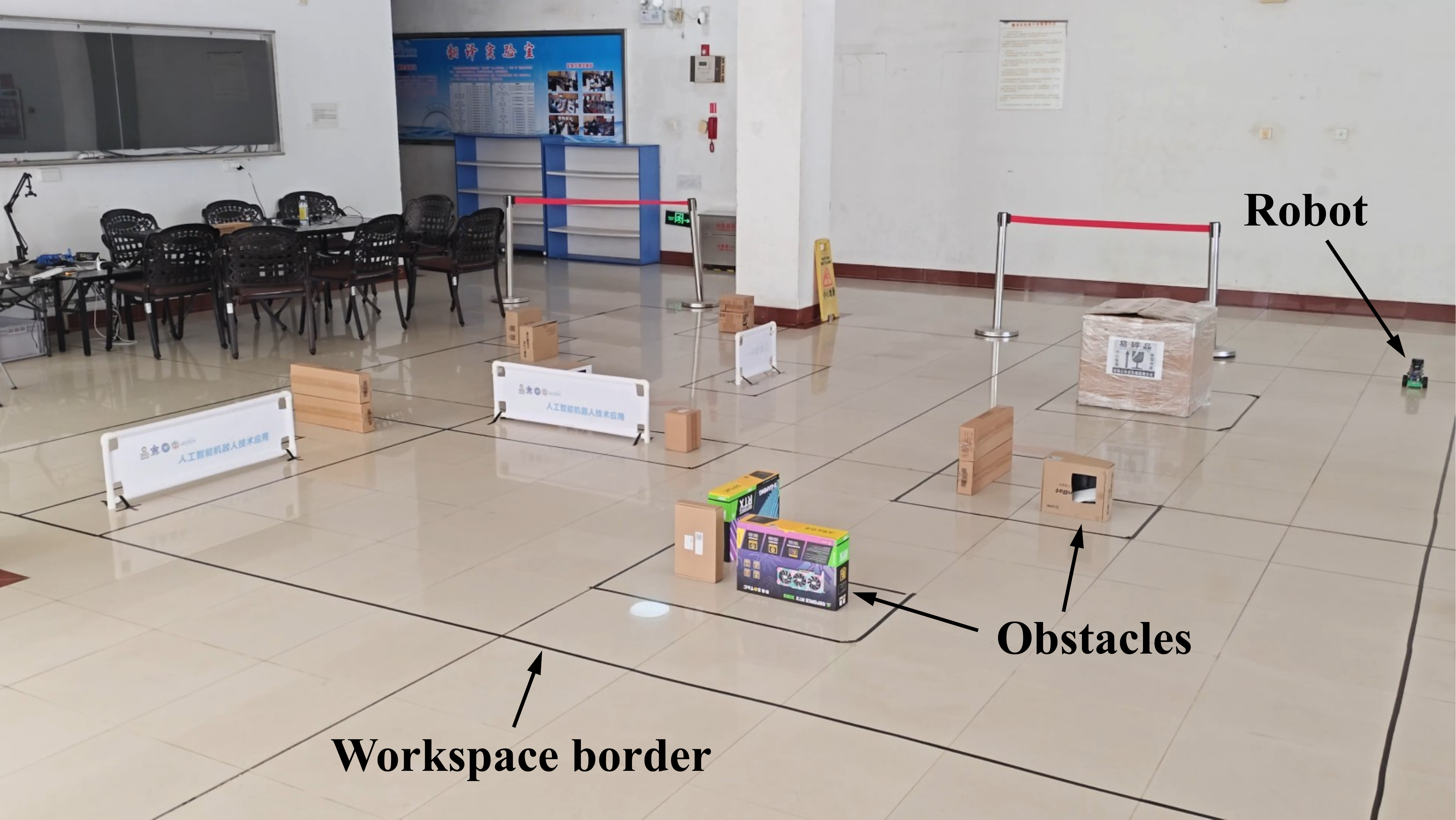}
    \caption{Scenario for real-robot experiment.}\label{fig:experiment_scenario} 
    \vspace{-1.0em}
 \end{figure}

\subsubsection{Comparison with the Offline TSP Solution}

We further assess the path quality of TRACE against an offline TSP solution. For each test environment, a cell-level TSP is constructed using the full map and solved exactly with the Lin-Kernighan-Helsgaun algorithm~\cite{Helsgaun2000}. The resulting path length provides an offline reference for the shortest traversal under the adopted graph-distance formulation. As shown in Figs.~\ref{fig:simulation_res_office}-\ref{fig:simulation_res_mall_three}, the online paths generated by TRACE remain close to this reference across all six scenarios, despite operating without prior knowledge of the obstacle layout.

\subsubsection{Performance in the Presence of Uncertainties}

To examine the robustness of TRACE to imperfect onboard perception, additional simulations are conducted with noise applied to range sensing, heading estimation, and localization. These perturbations affect both obstacle placement in the map and the estimated robot pose used for online planning. Independent zero-mean Gaussian noise is added to the corresponding measurements. The range and heading noise levels are fixed at $\sigma_{\mathrm{range}}=1.5\,\mathrm{cm}$ and $\sigma_{\mathrm{heading}} = 0.5^\circ$, respectively, while localization uncertainty is varied from $\sigma_{\mathrm{localization}} =0.02\,\mathrm{m}$ to $0.12\,\mathrm{m}$~\cite{shen2026cstar}. Robustness is evaluated using the actual coverage rate, computed from the ground-truth robot trajectory and obstacle map. The coverage rate is defined as the fraction of ground-truth obstacle-free cells reached by the robot's coverage footprint during execution. Fig.~\ref{fig:CoverageRatio_obsMap2_obsMap3} summarizes the results over $10$  Monte Carlo trials at each localization-noise level for warehouse and forest scenarios. Each box reports the distribution of the resulting coverage rates, with the center line indicating the median and the lower and upper box boundaries corresponding to the $25$th and $75$th percentiles, respectively.

\subsection{Validation by Real Experiments}

The performance of TRACE is further evaluated through real-robot experiments. As shown in Fig~\ref{fig:experiment_scenario}, the scenario is conducted in a $7\,\mathrm{m}\times7\,\mathrm{m}$ laboratory space with multiple obstacles. The mobile robot is equipped with an RPLIDAR S2L with a sensing range of $8$ m for obstacle detection and an MPU9250 IMU for estimating velocity, acceleration, and orientation. The gmapping algorithm~\cite{grisetti2007improved} is used for localization. All sensing, mapping, localization, and coverage planning are performed on an NVIDIA Jetson Nano. 

TRACE is compared with the same five baseline algorithms used in the simulations. All these methods are evaluated using the same robot platform, map resolution, and initial position. For each method, the robot incrementally maps the initially unknown environment and generates its coverage path until complete coverage is achieved. Fig.~\ref{fig:experiment_res} shows the coverage paths generated by TRACE and the baseline algorithms. As seen, TRACE produces more coherent trajectories with fewer long revisiting segments and less repeated traversal than the baseline methods. This figure also provides quantitative comparison results in terms of coverage time, path length, overlap ratio, and number of turns. Overall, TRACE achieves significant improvements over the baseline algorithms in all metrics, thus validating its effectiveness on real robots.

\section{Conclusions and Future Work} \label{sec:conclusions}

This paper presents a novel online CPP algorithm, called TRACE, for real-time coverage of unknown environments. A hierarchical coverage tree is constructed to globally represent the evolving connectivity of the uncovered space as new obstacles are discovered and coverage progresses. Based on the updated tree, an incremental global tour is maintained by locally refining only the portions affected by tree expansion while preserving the visiting order of unchanged regions. Guided by this tour, the local planner performs efficient dead-end recovery and generates global-tour-aware coverage paths to reduce redundant travel. Theoretical analysis established the computational complexity and complete coverage property of TRACE, and provided an approximation bound for the incremental global tour refinement. Extensive simulations and real-world robot experiments demonstrated that TRACE consistently improves coverage efficiency over existing online CPP methods in terms of coverage time, path length, overlap ratio, and number of turns. Future work will investigate extensions of TRACE to dynamic environments~\cite{shen2023smart,shen2026motion} and coverage problems in three-dimensional environments~\cite{ou2025pl}.

\balance
\bibliographystyle{IEEEtran}
\bibliography{reference}

\end{document}